%% file: main.tex
\documentclass{article}

    \usepackage[final,nonatbib]{neurips_2022}

\usepackage[numbers]{natbib}
\usepackage[utf8]{inputenc} % allow utf-8 input
\usepackage[T1]{fontenc}    % use 8-bit T1 fonts
\usepackage{hyperref}       % hyperlinks
\usepackage{url}            % simple URL typesetting
\usepackage{booktabs}       % professional-quality tables
\usepackage{amsfonts}       % blackboard math symbols
\usepackage{amssymb}
\usepackage{amsmath}
\usepackage{nicefrac}       % compact symbols for 1/2, etc.
\usepackage{microtype}      % microtypography
\usepackage{xcolor}         % colors
\usepackage{enumitem}
\usepackage{wrapfig}
\usepackage{graphicx}
\usepackage{caption}
\usepackage{amsthm}
\usepackage{multirow}
\usepackage[export]{adjustbox}
\usepackage{dsfont}

\usepackage[algoruled,boxed,lined,noend]{algorithm2e}
\graphicspath{{images/}}

\DeclareMathOperator*{\argmax}{arg\,max}

\theoremstyle{plain}
\newtheorem{theorem}{Theorem}[section]
\newtheorem{lemma}[theorem]{Lemma}

\title{When to Ask for Help: Proactive Interventions in Autonomous Reinforcement Learning}
\author{%
  Annie Xie\thanks{equal contribution}, Fahim Tajwar\footnotemark[1], Archit Sharma\footnotemark[1], Chelsea Finn \\
  Department of Computer Science\\
  Stanford University\\
  \texttt{\{anniexie,tajwar93,architsh,cbfinn\}@stanford.edu}
}

\begin{document}

\maketitle

%%CF.4.30: General recommendation: I'd suggest using numbered citations rather than author, year. The numbered citations will save you a fair amount of space, and improve readability. Author citations have merits too, I think the numbered citations outweigh it. It's good to make this decision early, since I think it will affect how cite works.

\begin{abstract}
A long-term goal of reinforcement learning is to design agents that can autonomously interact and learn in the world. A critical challenge to such autonomy is the presence of irreversible states which require external assistance to recover from, such as when a robot arm has pushed an object off of a table. While standard agents require constant monitoring to decide when to intervene, we aim to design proactive agents that can request human intervention only when needed. To this end, we propose an algorithm that efficiently learns to detect and avoid states that are irreversible, and proactively asks for help in case the agent does enter them. On a suite of continuous control environments with unknown irreversible states, we find that our algorithm exhibits better sample- and intervention-efficiency compared to existing methods. Our code is publicly available at \url{https://sites.google.com/view/proactive-interventions}.
\end{abstract}

\input{introduction}

\input{related_work}

\input{framework}
\input{experiments}

\input{conclusion}

\begin{ack}
AX was supported by an NSF Graduate Research Fellowship. The work was also supported by funding from Google, Schmidt Futures, and ONR grants N00014-21-1-2685 and N00014-20-1-2675. The authors would also like to thank members of the IRIS Lab for helpful feedback on an early version of this paper.
\end{ack}

\bibliographystyle{plainnat}
\bibliography{references}

%%%%%%%%%%%%%%%%%%%%%%%%%%%%%%%%%%%%%%%%%%%%%%%%%%%%%%%%%%%%
\section*{Checklist}

\iffalse
%%% BEGIN INSTRUCTIONS %%%
The checklist follows the references.  Please
read the checklist guidelines carefully for information on how to answer these
questions.  For each question, change the default \answerTODO{} to \answerYes{},
\answerNo{}, or \answerNA{}.  You are strongly encouraged to include a {\bf
justification to your answer}, either by referencing the appropriate section of
your paper or providing a brief inline description.  For example:
\begin{itemize}
  \item Did you include the license to the code and datasets? \answerYes{See Section~\ref{gen_inst}.}
  \item Did you include the license to the code and datasets? \answerNo{The code and the data are proprietary.}
  \item Did you include the license to the code and datasets? \answerNA{}
\end{itemize}
Please do not modify the questions and only use the provided macros for your
answers.  Note that the Checklist section does not count towards the page
limit.  In your paper, please delete this instructions block and only keep the
Checklist section heading above along with the questions/answers below.
%%% END INSTRUCTIONS %%%
\fi

\begin{enumerate}

\item For all authors...
\begin{enumerate}
  \item Do the main claims made in the abstract and introduction accurately reflect the paper's contributions and scope?
    \answerYes{}
  \item Did you describe the limitations of your work?
    \answerYes{See Section~\ref{sec:conclusion}.}
  \item Did you discuss any potential negative societal impacts of your work?
    \answerYes{}
  \item Have you read the ethics review guidelines and ensured that your paper conforms to them?
    \answerYes{}
\end{enumerate}

\item If you are including theoretical results...
\begin{enumerate}
  \item Did you state the full set of assumptions of all theoretical results?
    \answerYes{} %%CF.5.19: It would be nice to include a justification here.
        \item Did you include complete proofs of all theoretical results?
    \answerYes{See Appendix~\ref{app:proofs}.}
\end{enumerate}

\item If you ran experiments...
\begin{enumerate}
  \item Did you include the code, data, and instructions needed to reproduce the main experimental results (either in the supplemental material or as a URL)?
    \answerYes{}
  \item Did you specify all the training details (e.g., data splits, hyperparameters, how they were chosen)?
    \answerYes{}
        \item Did you report error bars (e.g., with respect to the random seed after running experiments multiple times)?
    \answerYes{}
        \item Did you include the total amount of compute and the type of resources used (e.g., type of GPUs, internal cluster, or cloud provider)?
    \answerYes{}
\end{enumerate}

\item If you are using existing assets (e.g., code, data, models) or curating/releasing new assets...
\begin{enumerate}
  \item If your work uses existing assets, did you cite the creators?
    \answerYes{}
  \item Did you mention the license of the assets?
    \answerYes{}
  \item Did you include any new assets either in the supplemental material or as a URL?
    \answerYes{}
  \item Did you discuss whether and how consent was obtained from people whose data you're using/curating?
    \answerNA{}
  \item Did you discuss whether the data you are using/curating contains personally identifiable information or offensive content?
    \answerNA{}
\end{enumerate}

\item If you used crowdsourcing or conducted research with human subjects...
\begin{enumerate}
  \item Did you include the full text of instructions given to participants and screenshots, if applicable?
    \answerNA{}
  \item Did you describe any potential participant risks, with links to Institutional Review Board (IRB) approvals, if applicable?
    \answerNA{}
  \item Did you include the estimated hourly wage paid to participants and the total amount spent on participant compensation?
    \answerNA{}
\end{enumerate}

\end{enumerate}

%%%%%%%%%%%%%%%%%%%%%%%%%%%%%%%%%%%%%%%%%%%%%%%%%%%%%%%%%%%%

\input{appendix}

\end{document}

%% file: introduction.tex
\begin{figure*}[h]
    \centering
    % \vspace{-0.4cm}
    \includegraphics[width=\linewidth]{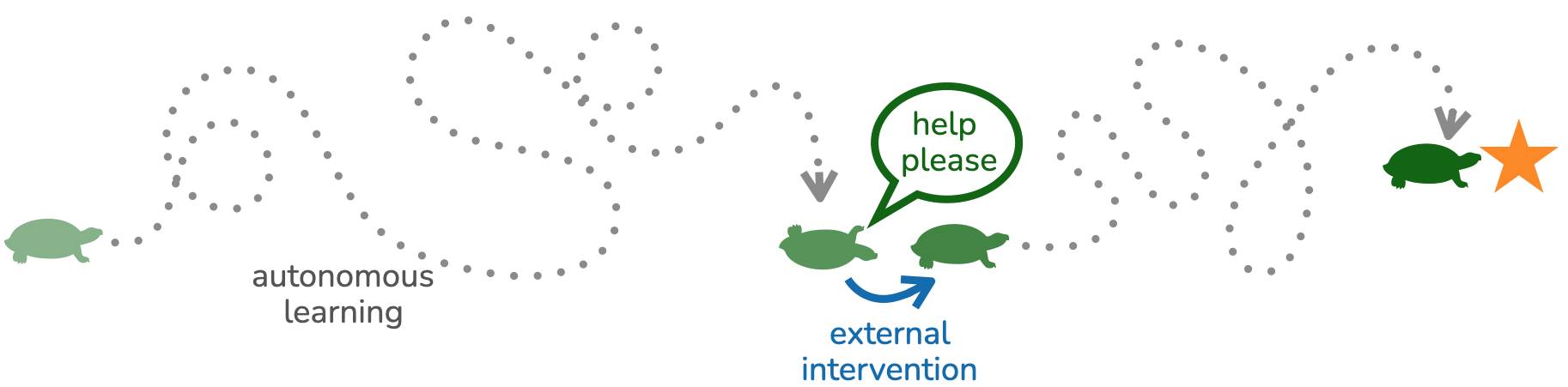}
    % \vspace{-0.5cm}
    \caption{\small Autonomous agents struggle to make progress without external interventions when they are stuck in an irreversible state. Reinforcement learning agents therefore need active monitoring throughout training to detect and intervene when the agent reaches an irreversible state. Enabling the agents to detect irreversible states and proactively request for help can substantially reduce the human monitoring required for training agents.}
    \label{fig:teaser}
    % \vspace{-0.1cm}
\end{figure*}

\section{Introduction}
\label{sec:introduction}

% AS.5.18: I am not sure how, but I think it is useful to reiterate why reducing human supervision is necessary or how it scales things
A reinforcement learning (RL) agent should be able to autonomously learn behavior by exploring in and interacting with its environment. However, in most realistic learning environments, there are irreversible states from which the agent cannot recover on its own. For example, a robot arm can inadvertently push an object off the table, such that an external supervisor must return it back to the robot’s workspace to continue the learning process. Current agents demand constant monitoring to decide when the agent enters an irreversible state and therefore when to intervene. In this work, we aim to build greater autonomy into RL agents by addressing this problem. In particular, we envision proactive agents that can instead detect irreversible states, proactively request interventions when needed, and otherwise learn autonomously.

Prior works have studied autonomy in RL, aiming to minimize the number of human-provided resets at the end of each episode, but generally assume the environment is fully reversible~\citep{han2015learning,zhu2020ingredients,sharma2021autonomous,sharma2021autonomouscurr,Gupta2021ResetFreeRL}.
Our work focuses on settings with potential irreversible states and algorithms to avoid such states. A related
%%CF: It's not clear if you mean related to the previous paragraph or related to the previous sentence. Can you clarify related to what?
%%AX: Done.
desiderata, however, arises in the safe RL setting; safe RL methods aim to learn policies that minimize visits to unsafe states, and the developed approaches are designed to avoid those particular parts of the state space~\citep{achiam2017constrained,chow2017risk,tessler2018reward,srinivasan2020learning,wagener2021safe,thomas2021safe}.
%%CF: I feel like this sentence above is a bit out of place and breaks the flow. Would it make sense to put it somewhere else, e.g. incorporate the information into the next paragraph?
%%AX: Moved the sentence to the beginning of the next para!
Prior safe RL algorithms assume that the agent is given a safety label on demand for \textit{every} state it visits. In contrast, an autonomous agent may not know when it has reached an irreversible state (such as knocking an important object off a table), and an algorithm in this setting should instead learn to both detect and avoid such states, while minimizing queries about whether a state is reversible.

%One way to reduce the supervision is by querying labels only at the end of each trial, instead of at each time-step.
%%CF.5.19: I think that this is going to sound like a TON of labels, because normally trials are pretty short. (which is pretty different from our setting). I think that the old version didn't have this issue. (addressed)
With this in mind, we design our setup to provide help to the agent in two ways: through an environment reset or through the reversibility label of a state.
However, unlike in safe RL, we can reduce the labeling requirement with a simple observation: all states \emph{proceeding} an irreversible state are irreversible, and all states \emph{preceding} a reversible state will be reversible. Based on this observation, we design a scheme based on binary search to generate reversibility labels for a trajectory of length $T$ using at most $\mathcal{O}(\log T)$ label queries, compared to the $\mathcal{O}(T)$ queries made by safe RL methods. We further reduce labeling burden by only querying labels in a large batch at the end of each extended episode, i.e. typically only after tens of thousands of steps. By combining this label efficient scheme with proactive requests for an intervention and batch of labels, we can enable agents to learn amidst irreversible states with a high degree of autonomy. %, without constant human monitoring and supervision.

Concretely, we propose a framework for reversibility-aware autonomous RL, which we call \emph{proactive agent interventions (PAINT)}, that aims to minimize human monitoring and supervision required throughout training. First, we train a reversibility-aware $Q$-value function that penalizes visits to irreversible states. Second, the reversibility labels are generated by a label-efficient binary search routine, which makes at most a logarithmic number of queries in the length of the interaction with the environment. Finally, the labeled states can be used to learn a classifier for predicting irreversible states, which can then be leveraged to proactively call for interventions. Our proposed framework PAINT can be used to adapt any value-based RL algorithm, in both episodic and non-episodic settings, to learn with minimal and proactive interventions.
%%CF.5.17: the length of which trajectory? maybe "our framework leverages only a logarithmic number of queries in the length of the agent's experience"? I guess it's actually less than that in some cases. so maybe "leverages at most a logarithmic number of..."
%%CF.5.17: don't call it an oracle.
We compare PAINT to prior methods for autonomous RL and safe RL on a suite of continuous control tasks, and find that PAINT exhibits both better sample- and intervention-efficiency compared to existing methods. On challenging autonomous object manipulation tasks, PAINT only requires around $100$ interventions while training for $3$ million steps, which is up to $15\times$ fewer than those required by prior algorithms.

% \archit{add more concrete and exciting experimental details}
%%CF.5.19: +1 I think that would be nice.

%% file: related_work.tex
\section{Related Work}
\label{sec:related_work}

Deployment of many RL algorithms in physical contexts is challenging, because they fail to avoid undesirable states in the environment and require human-provided resets between trials. Safe RL, reversibility-aware RL, and autonomous RL, which we review next, address parts of these problems.

\textbf{Safe RL.} 
The goal of our work is to learn to avoid irreversible states. Algorithms for safe RL also need to avoid regions of the state space, and achieve this by formulating a constrained optimization problem~\citep{chow2017risk,tessler2018reward,srinivasan2020learning,zanger2021safe} or by assigning low rewards to unsafe states~\citep{wagener2021safe,thomas2021safe}. 
% However, these algorithms optimize policies that are eventually safe, at the end of training, whereas we are interested in algorithms that minimize constraint violations throughout training. 
Another class of algorithms construct shielding-based policies that yield control to a backup policy if following the learning policy leads to an unsafe state~\citep{achiam2017constrained,alshiekh2018safe,turchetta2020safe,thananjeyan2020recovery,bharadhwaj2020conservative,wagener2021safe,bastani2021safe}. %and planning-based algorithms that produce safe actions under a learned safety model~\citep{thananjeyan2020safety}. 
% \annie{TODO: depends on the final method}
% In our method, a forward policy and reset policy similarly defer control to one another in a shielding-based fashion.
%%CF.5.12: I don't think that this is very similar, because the reset policy is going back to the initial state / trying to provide good initial states, whereas the recovery policy in prior works is trying to just get away from unsafe states. I think that calling them similar may be confusing.
Critically, however, all of these approaches assume that safety labels for each state can be queried freely at every time-step of training, whereas our objective is to minimize labeling requirements over training.

\textbf{Reversibility-aware RL.} Reversibility and reachability have been studied in the context of RL to avoid actions that lead to irreversible states~\citep{kruusmaa2007don,krakovna2018penalizing,rahaman2020learning,grinsztajn2021there} or, conversely, to guide exploration towards difficult-to-reach states~\citep{savinov2018episodic,badia2020never}. 
% However, prior work has only studied reversibility in the context of episodic RL settings. We argue that estimating and leveraging reversibility is more relevant to the non-episodic,
%%CF.4.30: I don't think we need to make this argument, especially since it comes across as critical of these prior works. I think that instead we can say "Unlike these works, our motivation lies in building greater autonomy in RL systems, and thus we focus on non-episodic settings where the goal is to minimize the number of human interventions during learning."
%%CF.4.30: Though, still, I think we need to do more to distinguish the differences, or else provide a comparison to one of these methods.
% reset-free setting, in which entering an irreversible state results in a costly intervention.
% ~\cite{grinsztajn2021there} propose a self-supervised approach for estimating reversibility from data, which we show can be used in conjunction with our proposed framework.
%%CF.4.30: this makes it sound like we are doing exactly what they are doing. revise?
Unlike prior work, our study of reversibility primarily focuses on the non-episodic setting to minimize the number of human interventions during learning. 
While prior methods are self-supervised, our experiments also find that our algorithm learns with significantly fewer interventions than prior methods by leveraging some binary reversibility labels.
%Previously proposed algorithms also estimate reversibility through self-supervised approaches, whereas our algorithm leverages a small number of binary labels indicating reversibility. In our experiments, we find our algorithm learns reversible behavior significantly faster with this supervision.

%%CF.4.30: the transition here is rather abrupt
\textbf{Autonomous RL.} Multiple prior works have also studied autonomy in RL, motivated by the fact that deployments of RL algorithms on real robots often require human-provided resets between episodes~\citep{finn2016deep,gu2017deep,ghadirzadeh2017deep}. To avoid the supervision needed for episodic resets, prior work has proposed to learn controllers to return to specific state distributions, such as the initial state~\citep{han2015learning,eysenbach2017leave}, the uniform distribution over states~\citep{zhu2020ingredients} or demonstration states~\citep{sharma2022state}, adversarially learned distributions~\citep{Xu2020ContinualLO}, or curriculum-based distributions~\citep{sharma2021autonomouscurr}.
However, most work in the reset-free setting assumes the agent's environment is reversible~\citep{moldovan2012safe,sharma2021autonomous,sharma2021autonomouscurr,Gupta2021ResetFreeRL}, whereas we specifically tackle the setting where this is not the case. 
% We instead rely on the agent to learn to detect whether it is in an irreversible state.
One notable exception is the Leave No Trace algorithm~\citep{eysenbach2017leave}, which checks whether the agent has successfully returned back to the initial state distribution and requests an intervention otherwise.
Our approach differs from Leave No Trace by requesting a reset based on the estimated reversibility of the state, which we find requires significantly fewer interventions in our evaluation.

\textbf{Human-in-the-loop learning.} Learning from human feedback has enabled RL agents to acquire complex skills that are difficult to encode in a reward function~\citep{knox2009interactively,macglashan2017interactive,wang2018interactive,faulkner2020interactive,arzate2020survey}. However, interactive RL algorithms are often to difficult to scale as they rely on feedback at every time-step of training. More feedback-efficient algorithms have learned reward models from human-provided preferences~\citep{akrour2011preference,sugiyama2012preference,wirth2013preference,sadigh2017active,biyik2018batch,lee2021pebble,wang2022skill}, which removes the need for constant feedback. Similarly, the interactive imitation learning learning has seen more query-efficient algorithms, which query expert actions based on the estimated risk or novelty of a visited state~\citep{zhang2016query,menda2019ensembledagger,hoque2021lazydagger,hoque2021thriftydagger}. While these algorithms augment the agent with human-provided preferences or expert actions, our approach leverages a different mode of feedback, that is, reversibility labels for visited states. 

% \textbf{Constraint inference in inverse RL.} Prior work has shown how demonstrations, in addition to specifying the task at hand, can also convey constraints in the environment~\citep{scobee2019maximum,malik2021inverse,fischer2021sampling,stocking2021discretizing}. In particular, given a nominal model of the environment and demonstrations that do not violate the constraints, we can infer the safety constraints that maximizes the likelihood of observing the given demonstrations. In our experimental evaluation, we also seed our algorithm with a small set of demonstrations and similarly exploit the safety property of the visited states. Specifically, we can designate the states in the demonstrations as examples of reversible states to train a supervised classifier that estimates reversibility.
% \annie{TODO: E-stop paper}

%% file: framework.tex
\section{Reinforcement Learning in Irreversible Environments}

%%CF.5.12: please make a more descriptive section header. what problem are you describing?

% AS: I think we need to eschew the usage of safe states, and focus on irreversibility / non-ergodicity. They are related but have slightly different connotations.
Consider a Markov decision process ${\mathcal{M} = (\mathcal{S}, \mathcal{A}, \mathcal{P}, r, \rho_0, \gamma)}$ with state space $\mathcal{S}$, action space $\mathcal{A}$, transition dynamics $\mathcal{P}: \mathcal{S} \times \mathcal{A} \times \mathcal{S} \mapsto [0, 1]$, bounded reward function $r: \mathcal{S} \times \mathcal{A} \mapsto [R_{\textrm{min}}, R_{\textrm{max}}]$, initial state distribution $\rho_0: \mathcal{S} \mapsto [0, 1]$ and discount factor $\gamma \in [0, 1)$. In this work, we build on the formalism of autonomous RL \citep{sharma2021autonomous}, but we remove the 
% (implicit) 
assumption that the environment is reversible, i.e., the MDP is no longer strongly connected (see example in ~\citep[Chapter~38]{lattimore2020bandit} and description below).
%%CF.5.12: The related work section used the terminology "ergodic" and "reversible", and here you are using "strongly connected". It would be nice for the terminology to be consistent. I also think that it would be good to explain what strongly connected means, though I think you describe it later.
The environment is initialized at $s_0 \sim \rho$ and an algorithm continually interacts with the environment till it requests the environment to be reset via an external intervention to state $s_0' \sim \rho$.
%%CF.5.19: I modified the above to avoid the term oracle.
Specifically, an algorithm ${\mathbb{A}:\{s_i, a_i, s_{i+1}, r_i\}_{i=0}^{t-1} \mapsto (a_t, \pi_t)}$ generates a sequence ${(s_0, a_0, s_1, \ldots) }$ in $\mathcal{M}$,  mapping the states, actions, and rewards seen till time $t-1$ to an action $a_t \in \mathcal{A} \cup \{a_\texttt{reset}\}$, and the current guess at the optimal policy ${\pi_t: \mathcal{S} \times \mathcal{A} \mapsto [0, \infty)}$. Here, $a_\texttt{reset}$ is a special action the agent can execute to reset the environment through extrinsic interventions, i.e. $\mathcal{P}\left(\cdot \mid s, a_{\texttt{reset}}\right) = \rho_0(\cdot)$.
% \annie{Since intervention usually means switching to a backup policy in the safe RL literature, I'm wondering if we should adopt a different term to avoid confusion? If we're only considering reset interventions, we could just call them resets, or hard resets as in LNT.}
% \archit{Generally agree that it might be useful to have a different name: (a) it might be useful to expand the set to include querying state reversibility independent of resets in this work, (b) future work could consider a broader set of extrinsic actions?}

\begin{wrapfigure}{r}{0.24\linewidth}
    \vspace{-0.3cm}
    \centering
    \includegraphics[width=\linewidth]{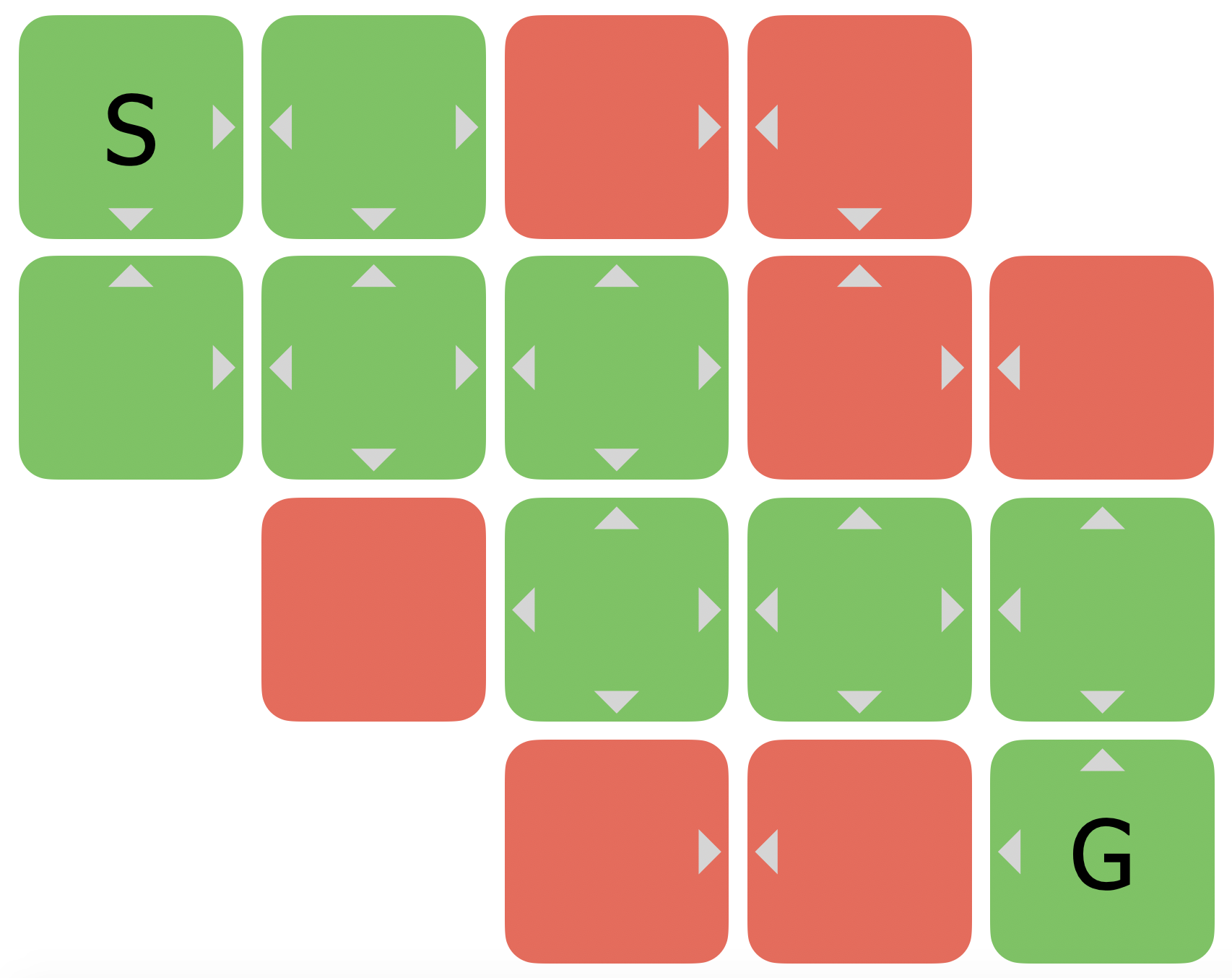}
    % \vspace{-0.5cm}
    \caption{\small Example of an MDP with irreversible states (in red). The agent starts in the state `S' and its goal is to reach the state `G', which are connected.}
    \vspace{-0.4cm}
    \label{fig:example}
\end{wrapfigure}
% Prior works have implicitly assumed that the MDP is strongly connected. 
%%CF.5.12: you already said this... also, it's pretty unclear what prior works you are referring to
A MDP is strongly connected if for all pairs of states $s_i, s_j \in \mathcal{S}$, there exists a policy $\pi$ such that $s_j$ has a non-zero probability of being visited when executing the policy $\pi$ from state $s_i$. This assumption can easily be violated in practice, for example, when a robot arm pushes an object out of its reach. At an abstract level, the agent has transitioned into a component of MDP that is not connected with the high reward states,
%%CF.5.12: I'm not sure if I like the terminology "connected", since connection is usually symmetric. in our case, there is a one-way connection, so it's not disconnected persay?
and thus cannot continue making progress, as visualized in Figure~\ref{fig:example}. The agent can invoke an extrinsic agent (such as a human supervisor)
%%CF.5.12: I think Annie was using the terminology "expert" before. Though, oracle does seem more fitting since an expert often refers to a demonstrator in the original action space (and a demonstrator also would not be able to reset). This kind of also makes me think that the interventions we are requesting aren't quite the same as a conventional reset, though perhaps it's fine to still call them resets.
through $a_{\texttt{reset}}$, and the extrinsic agent can reset the environment to a state from the initial state distribution. For example, the human supervisor can reset the object to the initial state, which is within the reach of the robot arm. For every state $s \in \mathcal{S}$, define $\mathcal{R}_\rho: \mathcal{S} \mapsto \{0,1\}$ as the indicator whether the state $s$ is in the same component as the initial state distribution.
State $s$ is defined to be reversible if $\mathcal{R}_\rho(s) = 1$, and
% i.e., it is connected to the initial state distribution and defined as 
irreversible if $\mathcal{R}_\rho(s) = 0$. We assume that the $\mathcal{R}_\rho$ is unknown, but can be queried for a state $s$.

While we do not assume that the MDP is strongly connected, we assume that the states visited by the optimal policy are in the same connected component as the initial state distribution. Under this assumption, we can design agents that can autonomously practice the task many times. Otherwise, the environment would need to be reset after every successful trial of the task.

The objective is to learn an optimal policy $\pi^* \in \argmax_\pi J(\pi) = \argmax_\pi \mathbb{E}[\sum_{t=0}^\infty \gamma^tr(s_t, a_t)]$.
%%CF: Note - I removed this phrase because it's already covered later (& covered in a better way) when talking about efficiency
Note, $J(\pi)$ is approximated by computing the return when the policy is rolled out from $s_0 \sim \rho_0$. Algorithms are typically evaluated on the sample efficiency, that is minimizing ${\mathbb{D}(\mathbb{A}) = \sum_{t=0}^\infty J(\pi^*) - J(\pi_t)}$.
%%CF.5.12: the number of interactions is much less important than the number of human interventions. Can we place much more emphasis on interventions than on number of env interactions?
However, since we care about minimizing the human supervision required and resetting the environment can entail expensive human supervision, we will primarily evaluate algorithms on \textit{intervention-efficiency}, defined as $\mathbb{I}(\mathbb{A}) = \sum_{k=0}^\infty J(\pi^*) - J(\pi_k)$, where $\pi_k$ is the policy learned after $k$ interventions. 

\section{Preliminaries}
\label{sec:prelims}
%%CF.5.12: more descriptive section header that includes preliminaries would be nice.

%%CF.5.12: This section covers way more preliminary info than they need, and sort of comes across as a related work section. Can you make sure to only include the information that is needed to understand the rest of the paper?
%%CF.5.12: It's also not clear why you are telling this to the reader (in part because you don't need to tell some of it to the reader). Though, explicitly saying how these methods will be used later would be helpful.
Episodic settings reset the environment to a state from the initial state distribution after every trial, typically after every few hundred steps of interaction with the environment. Such frequent resetting of the environment entails an extensive amount of external interventions, typically from a human. Prior works on autonomous RL have sought to reduce the supervision required for resetting the environments by learning a backward policy that resets the environment~\citep{eysenbach2017leave, zhu2019dexterous, sharma2021autonomouscurr}. Meaningfully improving the autonomy of RL in irreversible environments requires us to curb the requirement of episodic resets first. While our proposed framework is compatible with any autonomous RL algorithm, we describe MEDAL~\citep{sharma2022state}, which will be used in our experiments.
%%CF: Note I'm editing to (a) remove references to the word "oracle" per Ben's feedback and trying not to use the word supervision generically to refer to any human effort -- instead using more specific words like interventions.
% Prior works have demonstrated that standard RL methods developed in episodic settings do not recover an effective policy in non-episodic settings~\citep{zhu2020ingredients, co2020ecological, sharma2021autonomous}. Exploration strategies such as $\epsilon$-greedy or Gaussian noise can be sufficient in episodic settings where the trial always starts from the initial state distribution, but the state visitation can collapse around high-reward goal states leading to insufficient exploration. To ameliorate this, prior works explicitly learn a backward policy to enable the forward policy to repetitively try the task~\citep{eysenbach2017leave, zhu2020ingredients, sharma2021autonomouscurr}.

MEDAL learns a forward policy $\pi_f$ and a backward policy $\pi_b$, alternately executed for a fixed number of steps in the environment. The forward policy maximizes the conventional cumulative task reward, that is $\mathbb{E}\left[ \sum_{t=0}^\infty \gamma^t r(s_t, a_t) \right]$, and the backward policy minimizes the Jensen-Shannon divergence $\mathcal{D}_{\textrm{JS}} \left(\rho^b(s) \mid\mid \hat{\rho}^*(s) \right)$ between the marginal state distribution of the backward policy $\rho^b$ and the state distribution of the optimal forward policy $\hat{\rho}^*$, approximated by a small number of expert demonstrations.
%%CF.5.19: could we revise the above / shorten the above to simply say that the backward policy minimizes the Jensen-Shannon divergence between \rho^b and \hat{\rho}*, where \hat{\rho}* is the approximate state distribution of the optimal forward policy, approximated by a small number of expert demos? This would be shorter and it would make it absolutely clear that we don't need access to \rho*
Thus, the backward policy keeps the agent close to the demonstration states, allowing the forward agent to try the task from a mix of easy and hard initial states.
% Prior methods propose different objectives for the backward policy: assuming access to a reward function $r_\rho$ that encourages the agent towards the initial state distribution $\rho_0$, ~\citet{han2015learning} propose that $\pi_b$ maximize $\mathbb{E}\left[ \sum_{t=0}^\infty \gamma^t r_\rho(s_t, a_t) \right]$. This enables the forward policy $\pi_f$ to repeatedly try the task from the initial state distribution $\rho_0$. \citet{eysenbach2017leave} use the same objective for the backward policy, but additionally switch control from the forward policy $\pi_f$ to the backward
% % TODO: verify against LNT for correctness.
% policy $\pi_b$ whenever $Q^{\pi_b}(s, a) \leq \epsilon$ to ensure safe exploration, that is, whenever the probability of backward policy retrieving the agent to the initial state distribution falls below the threshold $\epsilon$. To encourage broader exploration, ~\citet{zhu2020ingredients} propose a perturbation backward controller $\pi_b$ that optimizes a novelty-based intrinsic reward~\citep{burda2018exploration} which forces the forward policy to solve the task from a wider initial state distribution.
The proposed divergence can be minimized via the objective %the following:
% \begin{equation}
%     \label{eq:minimax}
    $\min_{\pi_b} \max_C \mathbb{E}_{s \sim \rho^*}\big[\log C(s)\big] + \mathbb{E}_{s \sim \rho^b}\big[\log (1-C(s))\big]$,
% \end{equation}
where $C: \mathcal{S} \mapsto [0, 1]$ is a classifier that maximizes the log-probability of states visited in the forward demonstrations, and minimizes the probability of the states visited by the backward policy. The optimization problem for the backward policy can be written as a RL problem:
\begin{align}
    \min_{\pi_b} \mathbb{E}_{s \sim \rho^b} \left[\log\left(1 - C(s)\right) \right] = \max_{\pi_b} \mathbb{E}\left[-\sum_{t=0}^\infty \gamma^t \log\left(1 - C(s)\right)\right]
    \label{eqn:medal}
\end{align}
where $\pi_b$ maximizes the reward function $r(s, a) = - \log \left(1 - C(s)\right)$. Correspondingly, $C(s)$ is trained to discriminate between states visited by the backward policy and the demonstrations.

% AS: I am not giving any section names -- need to discuss them
\section{Proactive Agent Interventions for Autonomous Reinforcement Learning}
To minimize human monitoring and supervision when an agent is learning in an environment with irreversible states, the agent needs to (a) learn to avoid irreversible states over the course of training and (b) learn to detect and \textit{request} an intervention whenever the agent is stuck. For the former, we first describe a simple modification to the reward function to explicitly penalize visitation of irreversible states in Section~\ref{sec:framework}. However, such a modification requires the knowledge of reversibility of the visited states, which is not known apriori. We learn a classifier to estimate reversibility, proposing a label-efficient algorithm to query reversibility labels of visited states in Section~\ref{sec:label}. Since both the dynamics and the set of irreversible states are unknown apriori, the agent will inevitably still visit irreversible states as a part of the exploration. To ensure that a human does not have to monitor the agent throughout training, the agent should have a mechanism to decide and request for an intervention. We discuss such a mechanism in Section~\ref{sec:intervention}. Finally, we put together all these components in Section~\ref{sec:algorithm_summ} for our proposed framework \textbf{\underline{P}}roactive \underline{\textbf{A}}gent \underline{\textbf{INT}}erventions (PAINT), an overview of which is given in Figure~\ref{fig:framework}.

\begin{figure}
    \begin{minipage}{0.54\textwidth}
        \centering
        \includegraphics[width=\textwidth]{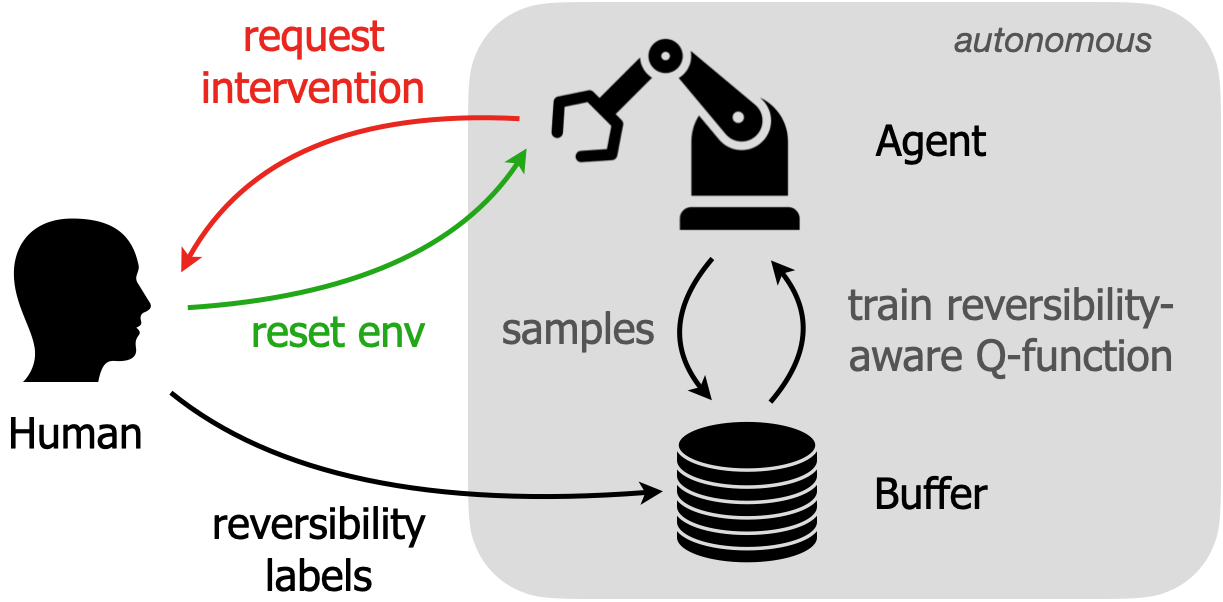}
        %%CF.5.19: Can we replace the word "Oracle" in the figure, with "Human" or something like that? (per Ben's feedback about an oracle being too strong of an assumption)
        \caption{\small Overview of our framework PAINT for minimizing human monitoring and supervision when learning in the presence of irreversible states. The agent proactively requests interventions, freeing the human from active monitoring of training. When an intervention is requested, the human resets the environment and provides reversibility labels for the latest experience since the previous intervention.}
        \vspace{-0.2cm}
        \label{fig:framework}
    \end{minipage}
    \hspace{0.04\linewidth}
    \begin{minipage}{0.44\textwidth}
        \small
        \centering
        \begin{algorithm}[H]
            \label{alg:bin_search}
            \SetAlgoLined
            \textbf{input:} $\tau = \{s_i\}_{i=0}^T$; {\color{olive}// unlabeled trajectory} \\
            % $\tau \leftarrow \tau_{\textrm{inp}}$\; 
            %%CF.5.19: \tau_inp is never used in the below. Is that a mistake, or can we remove this line and edit the above to call it \tau instead of \tau_inp? I think it might be a mistake, and you may need to change len(\tau) below to be len(\tau_inp), though not 100% sure.
            \While{$\textrm{len}(\tau) > 0$}{
                $m \leftarrow \lfloor\textrm{len}(\tau) / 2\rfloor$; {\color{olive}// get midpoint}\\
                {\color{olive}// query for midpoint}\\
                \If {$\mathcal{R}_\rho(s_m) = 1$}{
                    {\color{olive}// label first half reversible and query for second half}\\
                    label $\{s_i\}_{i=0}^m$ as 1\;
                    $\tau \leftarrow \{s_i\}_{i=m+1}^{\textrm{len}(\tau)}$\;
                }
                \Else{
                    {\color{olive}// label second half irreversible and query for first half}\\
                    label $\{s_i\}_{i=m+1}^{\textrm{len}(\tau)}$ as 0\;
                    $\tau \leftarrow \{s_i\}_{i=0}^{m}$\;
                }
            }
            \caption{Reversibility Labeling via Binary Search}
        \end{algorithm}
    \end{minipage}
\end{figure}

\subsection{Penalizing Visitation of Irreversible States}
\label{sec:framework}
Our goal is to penalize visitation of irreversible states by ensuring all actions leading to irreversible states are `worse' than those leading to reversible states. To this end, we adapt the reward-penalty framework from safe RL~\citep{thomas2021safe}
%%CF.5.19: This is a smaller comment, but it would be nice for the text to be clearer about which aspects of this come from Thomas et al. and which are new. 
for learning in the presence of irreversible states. 
%For simplicity of exposition, we will restrict the discussion to MDPs with deterministic transition dynamics, although the results extend to stochastic dynamics as discussed in Appendix \ref{app:proofs}.
% Let ${\mathcal{S_{\textrm{rev}}} = \{(s, a) \mid (\mathcal{R}_\rho(s') = 1\}}$ denote the set of state-action pairs that lead to a reversible state,
%%CF.5.19: this is assuming that the dynamics are deterministic? what if a sometimes leads to a reversible state and sometimes leads to an irreversible state? (e.g. due to action noise)
% where $\mathcal{R}_\rho$ indicates whether a state is reversible.
For a transition $(s, a, s')$, consider a surrogate reward function $\Tilde{r}$:
\begin{equation} \label{eq:surrogate_reward_function}
    \Tilde{r}(s, a) = \begin{cases}
        r(s, a), & \mathcal{R}_{\rho}(s') = 1\\
        R_{\textrm{min}} - \epsilon,  & \mathcal{R}_{\rho}(s') = 0
    \end{cases}
\end{equation}
Whenever the next state $s'$ is a reversible state, the agent gets the environment reward. Otherwise if it has entered an irreversible, it gets a constant reward $R_{\textrm{min}} - \epsilon$ that is worse than any reward given out by the environment. Whenever an agent enters an irreversible state, it will continue to remain in an irreversible state and get a constant reward of $R_{\textrm{min}} - \epsilon$. Therefore, the $Q$-value whenever $\mathcal{R}_{\rho}(s') = 0$ is given by:
\begin{align*}
    Q^\pi(s, a) = \mathbb{E}\left[\sum_{t=0}^\infty \gamma^t \Tilde{r}(s_t, a_t) \mid s_0 = s, a_0 =a \right] = (R_{\textrm{min}} - \epsilon)\sum_{t=0}^\infty \gamma^t = \frac{R_{\textrm{min}}-\epsilon}{1-\gamma}    
\end{align*}
%%CF: Note- added the intuition below, since we didn't describe it in the prev version.
This observation allows us to bypass the need to perform Bellman backups on irreversible states, and instead directly regress to the $Q$-value. More specifically, we can rewrite the loss function for the $Q$-value function as,
$
    \ell(Q) = \mathbb{E}_{(s, a, s', r) \sim \mathcal{D}}\left[Q(s, a) - \mathcal{B}^\pi Q(s,a)\right],
$
where the application of Bellman backup operator $\mathcal{B}^\pi Q(s, a)$ can be expanded as:
\begin{align}
    \mathcal{B}^\pi Q(s, a) &= \begin{cases}
        r(s, a) + \gamma \mathbb{E}_{a' \sim \pi(\cdot \mid s')}\hat{Q}(s', a'), & \mathcal{R}_\rho(s') = 1\\
        \left(R_{\textrm{min}} - \epsilon\right)/\left(1-\gamma\right), & \mathcal{R}_\rho(s') = 0
    \end{cases}\\
    &= \mathcal{R}_{\rho}(s')\left(r(s, a) + \gamma \mathbb{E}_{a' \sim \pi(\cdot \mid s')}\hat{Q}(s', a')\right) + \left(1-\mathcal{R}_\rho(s')\right) \frac{R_{\textrm{min}} - \epsilon}{1-\gamma}\label{eq:bellman_unsafe}
\end{align}
Here, $\mathcal{D}$ denotes the replay buffer, $\hat{Q}$ denotes the use of target networks commonly used in $Q$-learning algorithms to stabilize training when using neural networks as function approximators~\citep{mnih2015human}. This surrogate reward function and the modified Bellman operator can be used for any value-based RL algorithm in both episodic and autonomous RL settings. The hyperparameter $\epsilon$ controls how aggressively the agent is penalized for visiting irreversible states. In general, $Q$-values for actions leading to irreversible states will be lower than those keeping the agent amongst reversible states, encouraging the policies to visit irreversible states fewer times over the course of training. More details and proofs can be found in Appendix~\ref{app:q_val_proofs}.
% The following theorem shows that the surrogate reward function induces desirable behavior for $\epsilon > 0$:

% \begin{theorem} \label{thm:q_value_larger_for_reversible_states}
% Let $(s, a) \in \mathcal{S}_{\textrm{rev}}$ denote a state-action leading to a reversible state and let ${(s\_, a\_) \not\in \mathcal{S}_{\textrm{rev}}}$ denote a state-action pair leading to an irreversible state. Then, for all such pairs
% $$Q^\pi(s, a) > Q^\pi(s\_, a\_)$$
% for all $\epsilon > 0$ and for all policies $\pi$.
% \end{theorem}

% Appendix~\ref{app:q_val_proofs} provides additional discussion on . The result guarantees that actions leading to irreversible states indeed have worse $Q$-values than those keeping the agent amongst the reversible states, encouraging policies to visit irreversible states fewer times over the course of training. Additionally, any $\epsilon > 0$ should suffice in theory, though practical considerations for optimization affect this choice.
\vspace{-0.5mm}
\subsection{Estimating Reversibility}
\vspace{-0.5mm}
\label{subsection:empirical_reversability}
\label{sec:label}
%%CF: all good now I think
In general, $\mathcal{R}_\rho$ is not known apriori and will have to be estimated. We define ${\hat{\mathcal{R}}_\rho: \mathcal{S} \mapsto [0, 1]}$ as an estimator of the reversibility of a state $s\in\mathcal{S}$. We can then define an empirical Bellman backup operator $\hat{\mathcal{B}}^\pi$ from equation~\ref{eq:bellman_unsafe} by replacing $\mathcal{R}_\rho$ with the estimator $\hat{\mathcal{R}}_\rho$:
\begin{align*}
    \hat{\mathcal{B}}^\pi Q(s, a) = \mathbb{E}_{s'\sim \mathcal{P}(\cdot \mid s, a)}\left[\hat{\mathcal{R}}_{\rho}(s')\left(r(s, a) + \gamma \mathbb{E}_{a' \sim \pi(\cdot \mid s')}Q(s', a')\right) + \left(1-\hat{\mathcal{R}}_\rho(s')\right) \frac{R_{\textrm{min}} - \epsilon}{1-\gamma}\right]
    % \label{eq:bellman_empricial_unsafe_main}
\end{align*}
Analogous to $\hat{\mathcal{B}}^\pi$, we can define the empirical Bellman optimality operator $\hat{\mathcal{B}}^*$ (Eq~\ref{eq:bellman_optimal_unsafe}) for value iteration when using $\hat{\mathcal{R}}_\rho$. The following theorem bounds the suboptimality of the policy learned by value iteration under $\hat{\mathcal{B}}^*$:

\begin{theorem}\label{thm:final}
Let $\pi^*$ denote the optimal policy and $Q^*$ denote the corresponding optimal $Q$-value function. Let $\hat{\pi}^*$ denote the optimal policy returned by empirical Bellman optimality operator $\hat{\mathcal{B}}^*$. Assuming $\lVert \mathcal{R}_{\rho} - \hat{\mathcal{R}}_{\rho} \rVert_\infty \leq \delta$,
\begin{equation*}
    Q^{\hat{\pi}^*}(s, a) \geq Q^*(s, a) - \frac{2\delta \left(R_{\textrm{max}} - R_{\textrm{min}} + \epsilon\right)}{(1-\gamma)^2}
\end{equation*}

for all $(s, a) \in \mathcal{S} \times \mathcal{A}$.
\end{theorem}
\vspace{-2mm}
The proof and related discussion can be found in Appendix~\ref{app:emp_bell}. The result guarantees that closer $\hat{\mathcal{R}}_\rho$ is to $\mathcal{R}_\rho$ under the $\infty$-norm, the closer $\hat{\pi}^*$ is to $\pi^*$. To this end, we propose to learn $\hat{\mathcal{R}}_\rho$ by minimizing the binary cross-entropy loss ${\ell(\hat{\mathcal{R}}_\rho) = -\mathbb{E}_{s \sim \mathcal{D}}\big[\mathcal{R}_\rho(s) \log\hat{\mathcal{R}}_\rho(s) + (1-\mathcal{R}_\rho(s)) \log (1-\hat{\mathcal{R}}_\rho(s))\big]}$, where the states $s \sim \mathcal{D}$ represent the states visited by the agent.

Minimizing $\ell(\hat{\mathcal{R}}_\rho)$ requires the reversibility labels $\mathcal{R}_\rho(s)$ for $s \sim \mathcal{D}$. Since labeling requires supervision, it is critical to query $\mathcal{R}_\rho$ efficiently. Given a trajectory of states $\tau = (s_0, s_1, \ldots s_T)$, a na\"ive approach would be to query the labels $\mathcal{R}_\rho(s_i)$ for all states $s_i$, leading to $\mathcal{O}(T)$ queries per trajectory.
However, observe that we have the following properties: (a) all states \textit{following} an irreversible state will be irreversible and (b) all states \textit{preceding} a reversible state will be reversible. 
% AS: should this argument be included?
% Say $s_i$ is irreversible, that is, it is not connected with the initial state distribution $\rho_0$.
% If any future $s_{i+k}$ is reversible, i.e. a path exists from $s_{i+k}$ to some state from $\rho_0$, it would imply that $s_i$ is connected with the initial state distribution $\rho_0$ as a path exists from $s_i$ to $s_{i+k}$. This leads to a contradiction, showing us why the property (a) must be true. A similar argument can be constructed to show property (b).
It follows from these properties that every trajectory can be split into a reversible segment ${\tau_r = (s_0, s_1, \ldots s_k)}$ and an irreversible segment ${\tau_{\sim r} = (s_{k+1}, \ldots s_T)}$, where the irreversible segment $\tau_{\sim r}$ can be empty potentially. Identifying $s_{k+1}$, the first irreversible state, generates the labels for the entire trajectory automatically. Fortunately, we can construct a scheme based on binary search to identify $s_{k+1}$ in $\mathcal{O}(\log T)$ queries: $s_{k+1}$ occurs after the midpoint of the trajectory if the midpoint is reversible, otherwise it occurs before it. The pseudocode for this routine is given in Alg~\ref{alg:bin_search}.
%%CF.5.19: It would be nice to comment on the total number of labels over the complete training run, which would be \sum_{\tau\in\mathcal{D}} \log |\tau|. This is linear in the number of trajectories and logarithmic in the length of each trajectory. Hence, we would like a small number of very long trajectories rather than a large number of short trajectories. Fortunately, our overall algorithm will prefer the former over the latter because of the reward penalty. I think that this discussion is important because typical RL settings have a large number of short trajectories, which would not be very label efficient.

The total number of labels required would be $\mathcal{O}\left(N\log \left|\tau\right|_{\textrm{max}}\right)$, where $N$ is the number of trajectories in the replay buffer $\mathcal{D}$ and the $\left|\tau\right|_{\textrm{max}}$ denotes the maximum length of the trajectory. This represents a reduction in label requirement of $\mathcal{O}\left(N\left|\tau\right|_{\textrm{max}}\right)$ by prior safe RL methods. Furthermore, agent trains to avoid irreversible states, resulting in fewer and longer trajectories over the course of training. Thus, labeling reduces over time because the labels required is linear in $N$ and logarithmic in $\left|\tau\right|_{\textrm{max}}$.

% AS: This point is either very important to make or controversial that we should eschew. I think this should be 
% The main difference from the original reward-penalty framework is that we do not change the dynamics by adding absorbing states whenever the agent enters an irreversible state. Absorbing states change the dynamics, and assume that irreversible states can be detected automatically. The abstraction of absorbing states may be more realistic in context of safe RL where unsafe states might be easier to detect automatically because of difference in dynamics (ex: a broken robot), but reversible and irreversible states can be harder to detect automatically (ex: mug out of reach of robot arm).

\subsection{Proactive Interventions}

\label{sec:intervention}
Despite trying to avoid irreversible states via reward penalties, an agent will inevitably encounter some irreversible states due to exploratory behaviors. It is critical that the agent proactively asks for help in such situations, so that a human does not need to constantly monitor the training process. More specifically, an agent should request an intervention when it is an irreversible state. Since $\mathcal{R}_\rho(s)$ is not available, the agent again needs to estimate the reversibility of the state. It is natural to reuse the learned reversibility estimator $\hat{\mathcal{R}}_\rho$ for this purpose. We propose the following rule: the agent executes $a_{\texttt{reset}}$ whenever the reversibility classifier's prediction falls below 0.5, i.e., $\hat{\mathcal{R}}_\rho(s) < 0.5$.
% \begin{itemize}
    % \item
    % \textbf{Stuck Classifier}: We estimate $\mathcal{R}_\rho$ by fitting a classifier $\hat{\mathcal{R}}_\rho$ on the labeled data generated by Algorithm~\ref{alg:bin_search}. Whenever $\hat{\mathcal{R}}_\rho < 0.5$, we can request for an intervention.
%     \item \textbf{Value function}: Based on our assumption, we know that the optimal policy stays within the same connected component as the initial state distribution. Therefore, we can monitor the value function of the current state $V^\pi(s_t)$ and compare it to that of the initial state distribution. That is, if $V^\pi(s_t) + \eta < \mathbb{E}_{s \sim \rho_0}\left[V^\pi(s)\right]$, the agent can request an intervention, where $\eta$ is a hyperparameter to allow some slack.
% \end{itemize}

\subsection{Putting it Together}

\label{sec:algorithm_summ}
\begin{wrapfigure}{r}{0.5\linewidth}
    \vspace{-12mm}
    \small
    \begin{algorithm}[H]
        \label{alg:autonomousPAINT}
        \SetAlgoLined
        \textbf{input:} $\mathbb{P}$; {\color{olive}// agent, params abstracted away} \\ 
        initialize $\hat{\mathcal{R}}_\rho, \mathcal{D}$; {\color{olive} // rev classifier, replay buffer} \\
        \While{not done}{
            $s \sim \rho_0$; {\color{olive}// reset environment} \\
            {\color{olive}// continue till classifier detects irreversibility} \\
            \While{$\hat{\mathcal{R}}_\rho(s) > 0.5$}{
                {\color{olive}// step in the environment} \\
                $a \sim \mathbb{P}(s)$, $s \sim \mathcal{P}(\cdot \mid s, a)$\;
                {\color{olive}// update replay buffer and agent}\\
                update $\mathcal{D}, \mathbb{P}$\;                
            }
            {\color{olive}// optionally explore environment}\\
            \For{explore steps}{
                $a \sim \textrm{unif}(\mathcal{A})$, $s \sim \mathcal{P}(\cdot \mid s, a)$\;
                update $\mathcal{D}$\;
            }
            {\color{olive}// reversibility labels via binary search}\\
            update reversibility labels in $\mathcal{D}$\;
            {\color{olive}// train classifier on all labeled data, new and old}\\
            train $\hat{\mathcal{R}}_\rho$\;
        }
        \caption{PAINT}
    \end{algorithm}
    \vspace{-10mm}
\end{wrapfigure}
With the key components in place, we summarize our proposed framework. High-level pseudocode is given in Alg.~\ref{alg:autonomousPAINT}, and a more detailed pseudocode is deferred to Appendix~\ref{app:detailedpseudocode}.

PAINT can modify any value-based RL algorithm, in both episodic and autonomous settings. This description and Alg.~\ref{alg:autonomousPAINT} focus on the latter setting, although adapting it to the episodic setting is straightforward. The agent's interaction with the environment consists of a sequence of trials that end whenever the environment is reset to a state $s \sim \rho_0$. During each trial, the agent operates autonomously, and the Bellman update for the critic is modified according to the empirical Bellman backup $\hat{\mathcal{B}}^\pi$. %Eq~\ref{eq:emp_bellman_unsafe}. 
Whenever the reversibility classifier $\hat{\mathcal{R}}_\rho < 0.5$, parameterized as a neural network, the agent requests an intervention. The agent can execute a fixed number of exploration steps after requesting an intervention and before the intervention is performed. Whenever the classifier predicts an irreversible state correctly, these exploration steps can help the agent gather more information about irreversible states. At the time of the intervention, all new states visited since the previous intervention are labeled for reversibility via Algorithm~\ref{alg:bin_search}. Finally, the reversibility classifier is trained on all the labeled data before the environment is reset to a state $s \sim \rho_0$ for the next trial. Full implementation details can be found in Appendix~\ref{app:implementation}.

The agent is provided reversibility labels only when the external reset is provided. This simplifies supervision as the human can reset the environment and provide labels at the same time. This means the replay buffer $\mathcal{D}$ will contain states with and without reversibility labels, since states from the current trial will not yet have labels. We use Eq.~\ref{eq:bellman_unsafe} for states that have reversibility labels to avoid errors from the classifier affecting the critic update and use $\hat{\mathcal{B}}^\pi$ for those that do not have labels.

%% file: experiments.tex
\section{Experiments}
We design several experiments to study the efficiency of our algorithm in terms of the required number of reset interventions and number of queried reversibility labels. 
% In particular, we seek to answer the following:
% \begin{enumerate}[leftmargin=*,topsep=0pt,itemsep=0pt]
%     \item How does our algorithm compare to prior safe RL methods in terms of label-efficiency and to prior autonomous RL methods in terms of reset-efficiency?
%     \item How do different ways of handling the termination condition and unlabeled states affect the performance of our method?
%     % \item Can our algorithm generalize to noisy reversibility labels?
% \end{enumerate}
Our code and videos of our results are at: \url{https://sites.google.com/view/proactive-interventions}.

\subsection{Experimental Setup}

\textbf{Environments.} 
% We study four continuous control environments, each visualized in Fig.~\ref{fig:envs}.
To illustrate the wide applicability of our method, we design environments that represent three distinct RL setups: episodic, forward-backward, and continuing.
\begin{itemize}[leftmargin=*,topsep=0pt,noitemsep]
    \setlength{\itemsep}{0pt}
    \item \textbf{Maze} (episodic). A 2D continuous maze environment with trenches,
    %(the positions of which are unknown to the agent a priori)
    which represent groups of connected irreversible states. The agent can fall into a trench, and once entered, it can roam freely within the trench but cannot leave it without an environment reset. In this task, resets are infrequently provided to the agent after $500$ time-steps.
    
    \item \textbf{Tabletop Organization}~\cite{sharma2021autonomous} (forward-backward). The agent must grasp the mug and put it down on one of the four goal positions. Dropping the mug outside of the red boundary is irreversible.
    
    \item \textbf{Peg Insertion}~\cite{sharma2021autonomous} (forward-backward). The agent must insert the peg into the goal but can potentially drop it off the table, which is irreversible.
    
    \item \textbf{Half-Cheetah Vel}~\cite{brockman2016openai} (continuing). The agent must run at the specified target velocity, which changes every $500$ steps, and can potentially flip over onto its back, which is irreversible.
\end{itemize}

\begin{figure}
    \centering
    \includegraphics[width=0.95\linewidth]{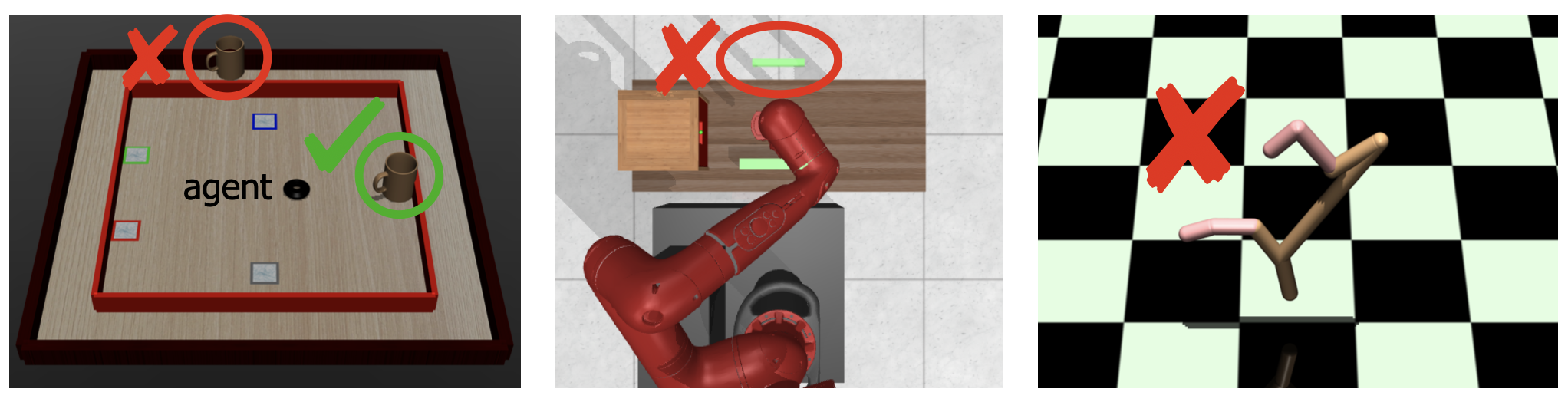}
    \caption{\small A subset of our evaluation tasks: Tabletop Manipulation, Peg Insertion, and Half-Cheetah Velocity. Irreversible states in the first two environments are when the agent drops the object outside the red boundary (\textit{left}) and off of the table (\textit{middle}). The cheetah is in an irreversible state whenever it is flipped over (\textit{right}).}
    \label{fig:envs}
    % \vspace{-0.5cm}
\end{figure}

We visualize and fully describe each environment in Fig.~\ref{fig:envs} and in Appendix~\ref{app:envs} respectively.

\textbf{Comparisons.} 
In the episodic and continuing settings, we consider safe RL baselines that rely on reversibility labels at every time-step of training.
\begin{itemize}[leftmargin=*,topsep=0pt,noitemsep]
    \setlength{\itemsep}{0pt}
    \item \textbf{Safe Model-Based Policy Optimization (SMBPO)}~\cite{thomas2021safe}. This comparison implements the modified Bellman operator defined in Eqn.~\ref{eq:bellman_unsafe} in Section~\ref{sec:framework}, using the true reversibility labels.

    \item \textbf{Safety Q-functions for RL (SQRL)}~\cite{srinivasan2020learning}. A safe RL method that trains a safety critic, which estimates the future probability of entering an irreversible state for a safety-constrained policy.
\end{itemize}
In the forward-backward setting, we consider methods designed for the autonomous learning setup. These methods do not require any reversibility labels. Hence, our goal here is to compare the reset-efficiency of our method to prior work.
\begin{itemize}[leftmargin=*,topsep=0pt,noitemsep]
    \setlength{\itemsep}{0pt}
    \item \textbf{Leave No Trace (LNT)}~\cite{eysenbach2017leave}. An autonomous RL method that jointly trains a forward policy and reset policy. When the reset policy fails to return to the initial state, the agent requests a reset.
    
    \item \textbf{Matching Expert Distributions for Autonomous Learning (MEDAL)}~\cite{sharma2022state}. This method trains a reset policy that returns to the distribution of demonstration states provided for the forward policy. MEDAL does not have a built-in intervention rule.
\end{itemize}

In all tasks, we compare to a recently proposed reversibility-aware RL method, \textbf{Reversibility-Aware Exploration (RAE)}~\cite{grinsztajn2021there}, which does not require any reversibility labels. It instead trains a self-supervised reversibility estimator to predict whether a state transition $(s, \tilde{s})$ is more likely than the reverse $(\tilde{s}, s)$. We augment RAE with an intervention rule, similar to our method, defined in terms of predictions from its self-supervised classifier. In the forward-backward setting, we train both the forward and backward policies with RAE. Finally, we also evaluate \textbf{Episodic RL}, which represents the typical RL setup with frequent resets and thus provides an upper-bound on task success. In Appendix~\ref{app:implementation}, we provide full implementation details of each comparison, and in Appendix~\ref{app:comparisons}, we discuss the set of assumptions made by each comparison.

\subsection{Main Results}
\label{subsection:main_results}

\begin{figure*}
\begin{minipage}{0.5\linewidth}
    \centering
    \includegraphics[width=0.58\linewidth,valign=t]{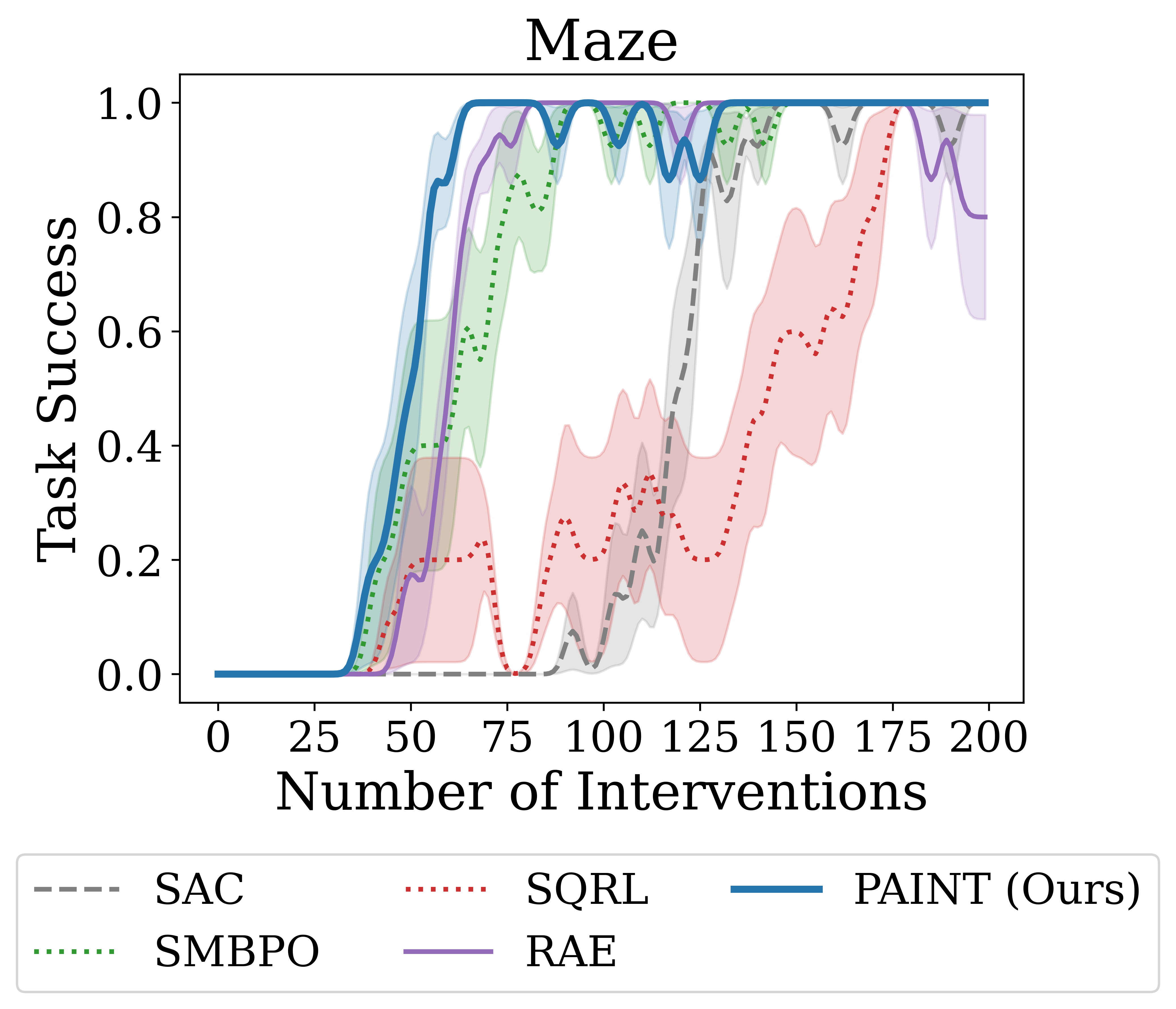} 
    \raisebox{-0.2cm}{\includegraphics[width=0.4\linewidth,valign=t]{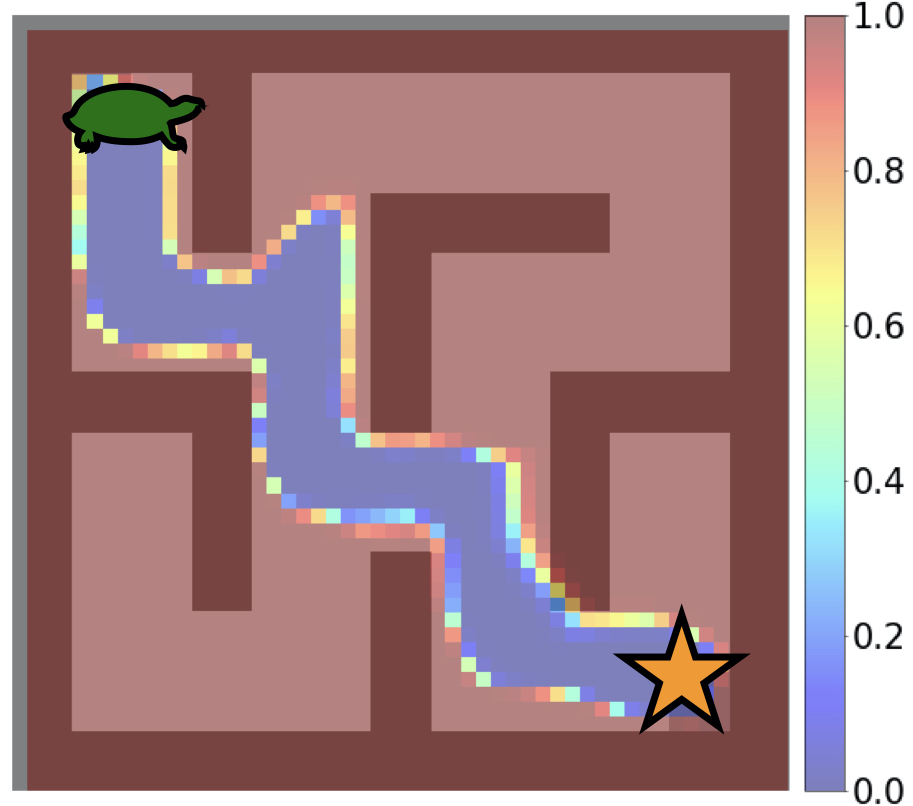}}
    \vspace{-0.1cm}
    \caption{\small (\textit{left}) Task success versus interventions. Shaded regions denote the standard error over $5$ seeds. (\textit{right}) Predictions generated by our reversibility classifier, where the purple region is predicted to be reversible.}
    \vspace{-0.6cm}
    \label{fig:maze_results}
\end{minipage}
\hspace{0.02\linewidth}
\begin{minipage}{0.48\linewidth}
    \vspace{0.6cm}
    \small
    \centering
    \setlength{\tabcolsep}{4pt}
    \begin{tabular}{llc}
        \toprule
        Task & Method & Labels \\
        \midrule
        \multirow{2}{*}{Maze} & SMBPO/SQRL & $200$K \\
                            %   & SQRL & $200$K \\
                              & PAINT (Ours) & $3260 \pm 12$ \\
                              \cmidrule{1-3}
        \multirow{1}{*}{Tabletop} & PAINT (Ours) & $1021 \pm 69$ \\
                              \cmidrule{1-3}
        \multirow{1}{*}{Peg Insertion} & PAINT (Ours) & $2083 \pm 149$ \\
                              \cmidrule{1-3}
        \multirow{2}{*}{Cheetah} & SMBPO w Term. & $3$M \\
                                          & PAINT (Ours) & $8748 \pm 3762$ \\
        \bottomrule
    \end{tabular}
    \caption{\small Number of queried reversibility labels. For our method, we average the number of labels used across $5$ seeds and report the standard error.}
    \label{tab:queries}
\end{minipage}
\end{figure*}
% \end{wraptable}
In Fig.~\ref{fig:maze_results} (\textit{left}) and Fig.~\ref{fig:main_results}, we plot the task success versus the number of interventions in the 4 tasks. For methods that use reversibility labels, we report the total number of labels queried in Table~\ref{tab:queries}. 

%%CF: Would it be possible to move details about the task set-up to be earlier, e.g. when describing the environment? (specifically the info about the reset frequency) I think this section should be dedicated to discussing the results, and not the experimental set-up.
%%AX: Done.
\textbf{Maze.} While the safe RL methods, SMBPO and SQRL, require reversibility labels at every time-step, our approach PAINT only requires on average $3260$ queries to label all $200$K states visited. In Fig.~\ref{fig:maze_results} (right), we visualize predictions from our reversibility classifier at the end of training, where zero predicts `reversible' and one predicts `irreversible'. The classifier correctly identifies the path that leads to the goal as reversible. Interestingly, it classifies all other regions as stuck states, including the states that \emph{are} reversible. Because these states are irrelevant to the task, however, classifying them as irreversible, and therefore to be avoided, is advantageous to our policy as it reduces its area of exploration.

\textbf{More complex domains.} In the Tabletop Organization and Peg Insertion tasks, each agent is reset every $200$K and $100$K time-steps, per the EARL benchmark~\citep{sharma2021autonomous}. However, we allow agents to request earlier resets, and under this setting, we compare PAINT to other methods that implement intervention rules. Compared to Leave No Trace and Reversibility-Aware Exploration, PAINT requires significantly fewer resets---$80$ and $124$ resets respectively, which corresponds to roughly \textbf{one intervention for every $\mathbf{25}$K steps}. Importantly, the number of interventions plateaus as training progresses, and the agent requires fewer and fewer resets over time (see Appendix~\ref{app:add_exps} for additional plots of number of interventions versus time-steps).
The exception is MEDAL (green segment near the origin), which is not equipped with an early termination rule and so only uses $10$ interventions total. However, it also fails to make meaningful progress on the task with few resets.

On the continuing Half-Cheetah task, agents do not receive any resets, unless specifically requested. Here, we compare PAINT to SMPBO with early termination, an \textbf{oracle} version of our method, which assumes that reversibility labels are available at \emph{every} time-step and immediately requests an intervention if the agent is flipped over. PAINT converges to its final performance after around $750$ resets, on par with the number of resets required by SMPBO with early termination. On the other hand, a standard episodic RL agent, which receives resets at every $2$K steps, and RAE, which trains a self-supervised classifier, learn significantly slower with respect to the number of interventions.

\begin{figure*}
    \centering
    \includegraphics[width=\linewidth]{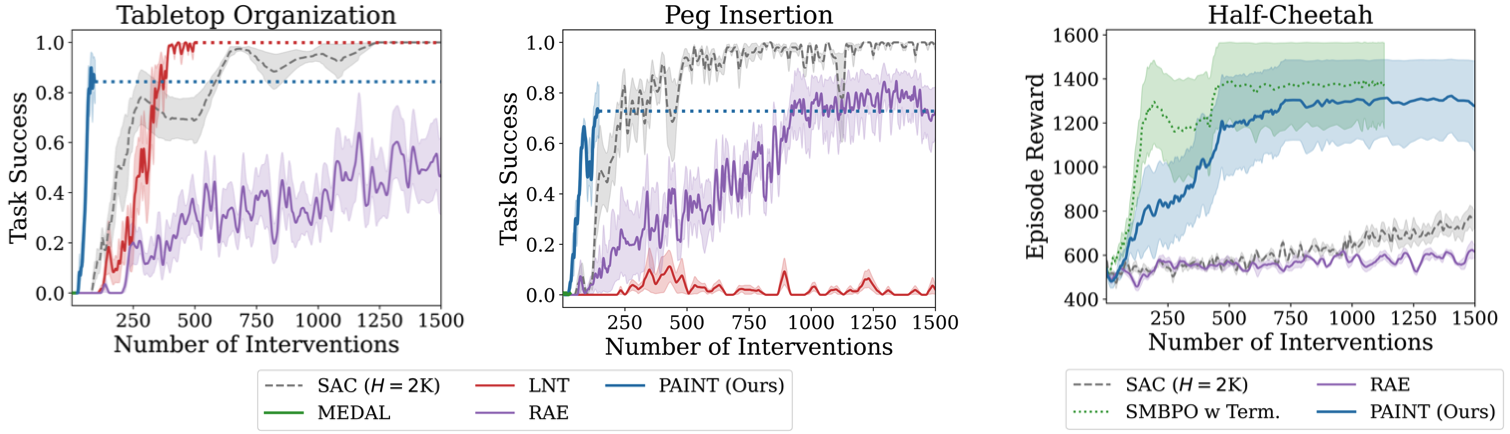}
    % \vspace{-0.45cm}
    \caption{\small Task success versus interventions averaged over $5$ seeds. %Shaded regions denote the standard error. 
    % Our approach PAINT requires significantly fewer resets than alternative approaches, with the exception of MEDAL~\citep{sharma2022state}.
    % MEDAL (green segment near origin) is not equipped with an early termination rule, receives resets only every $200$K time-steps, and so fails to solve the task. 
    Methods with stronger assumptions, i.e., SAC resets every $H$ steps and SMBPO requires labels at every time-step, are dotted.
    Note the short green segment near the origin representing MEDAL.
    % On Half-Cheetah, PAINT solves the task in fewer interventions than RAE, which does not require any reversibility labels, but more than SMBPO with oracle terminations, which requires labels at every single time-step.
    }
    \label{fig:main_results}
\end{figure*}

\subsection{Ablations and Sensitivity Analysis}
\label{subsection:ablations}

\begin{figure*}
    % \vspace{-0.3cm}
    \centering
    \includegraphics[width=0.66\linewidth]{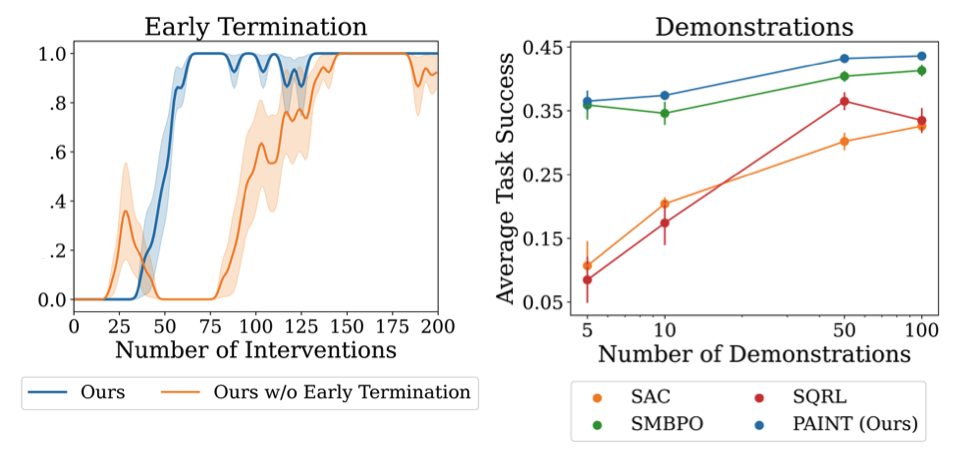}
    % \vspace{-0.3cm}
    \caption{\small (\textit{left}) After removing the early termination condition, which initiates random exploration, we find that PAINT learns less efficiently.
    % (\textit{middle}) We define a termination condition based on the Q-values, which learns with slightly fewer resets but also achieves slightly lower final performance.
    (\textit{right}) Varying the number of demonstrations suggests that PAINT and SMBPO are robust to the amount of available demonstrations.}
    % \vspace{-0.3cm}
    \label{fig:ablations}
\end{figure*}

\textbf{Early termination}. In the episodic Maze setting, our algorithm switches to a uniform-random policy for the remainder of the episode if the termination condition is met. In Fig.~\ref{fig:ablations} (\textit{left}), we plot the performance without early termination, i.e., running the agent policy for the full episode. Taking random explorations, after the agent believes it has entered an irreversible state, significantly helps our method, as it increases the number and diversity of irreversible states the agent has seen.

\textbf{Varying the number of demonstrations}. Our method leverages demonstrations in a subset of environments. While we provide these demonstrations to all comparisons as well, we want to study how much our method relies on them. We plot the average task success during training versus number of demonstrations in Fig.~\ref{fig:ablations} (\textit{right}). While PAINT and SMBPO are robust to the amount, alternative methods tend to achieve significantly lower success when given fewer demonstrations. 

% \textbf{Noisy reversibility labels}. Thus far, our method has assumed the reversibility labels are noiseless. However, these labels in most practical settings come from human supervisors, who can inadvertently introduce noise into the labeling process. Hence, we simulate noisy reversibility labels in the maze environment, and design a \emph{robust} variant of our labeling scheme to account for possible noise in the labels. In the robust variant, in addition to querying the label for a state $s$, we also query the neighboring states, i.e., the sequence of $N$ states centered at $s$, and take the majority as the label for $s$. 

% - Noisy reversibility indicator feature in the state representation

%% file: conclusion.tex
\section{Discussion}
\label{sec:conclusion}

In this work, we sought to build greater autonomy into RL agents, particularly in irreversible environments.
%In this work, we sought out to build greater autonomy into RL agents. We identified irreversible environments as a core obstacle to such autonomy, 
We proposed an algorithm, PAINT, that learns to detect and avoid irreversible states, and proactively requests an intervention when in an irreversible state. PAINT leverages reversibility labels to learn to identify irreversible states more quickly, and improves upon existing methods on a range of learning setups in terms of task success, reset-efficiency, and label-efficiency.

Despite these improvements, PAINT has multiple important limitations. In environments where irreversible states are not encountered until further into training, the reversibility classifier may produce false positives which would significantly delay the next intervention. Further, while PAINT is far more label-efficient than prior safe RL methods, it still requires around thousands of reversibility labels. We expect that this limitation may be mitigated with more sophisticated querying strategies, e.g. that take into account the classifier's confidence. Finally, we hope that future work can validate the ability for reversibility aware techniques to improve the autonomy of real robotic learning systems.
%\annie{other limitations and potential negative impacts} %To improve the accuracy of the classifier, we provided an initial set of diverse stuck states. \annie{other stuff}
%%CF: I added a discussion of limitations. For negative impacts, I can't think of anything that is specific to this work compared compared to general risks of RL systems (e.g. dual use). I think we could say no or n/a in the checklist and say that we were unable to identify any direct risks of negative societal benefits / any risks that are specific to the contributions of this paper. But, I'll leave that decision up to you.

%% file: appendix.tex
\newpage
\appendix

\section{Appendix}

\subsection{Proofs}
\label{app:proofs}

\subsubsection{Penalizing Visitation of Irreversible States}
\label{app:q_val_proofs}
The high level goal of this section is to show that $Q$-values prefer the actions leading to reversible states when using the surrogate reward function (Eq~\ref{eq:surrogate_reward_function}). We first present the result for deterministic dynamics, and then give a more restricted proof for stochastic dynamics. For deterministic dynamics, define $\mathcal{S}_{\textrm{rev}} = \{(s, a) \mid \mathcal{R}_\rho(s') = 1\}$.

\begin{theorem} Let $(s, a) \in \mathcal{S}_{\textrm{rev}}$ denote a state-action leading to a reversible state and let ${(s\_, a\_) \not\in \mathcal{S}_{\textrm{rev}}}$ denote a state-action pair leading to an irreversible state. Then, for all such pairs
$$Q^\pi(s, a) > Q^\pi(s\_, a\_)$$
for all $\epsilon > 0$ and for all policies $\pi$.
\label{thm:q_value_larger_for_reversible_states}
\end{theorem}

\begin{proof}
By definition, the reward function is bounded, i.e., $r(s', a') \in [R_{\textrm{min}}, R_{\textrm{max}}]$ for any ${(s', a') \in \mathcal{S} \times \mathcal{A}}$. For any ${(s', a') \in \mathcal{S}_{\textrm{rev}}}$, we have the following (using equation \ref{eq:surrogate_reward_function}):
$$\tilde{r}(s', a') = r(s', a') \geq R_{\textrm{min}} > R_{\textrm{min}} - \epsilon$$
and for any $(s', a') \notin \mathcal{S}_{\textrm{rev}}$, we have:
$$\tilde{r}(s', a') = R_{\textrm{min}} - \epsilon$$

This simplifies to $\tilde{r}(s', a') \geq R_{\textrm{min}} - \epsilon$ for any $(s', a') \in \mathcal{S} \times \mathcal{A}$. From the definition of $Q^\pi$, we have
\begin{equation}\label{eq:q_value_lower_bound}
    \begin{aligned}
    Q^\pi(s', a') = {} & \mathbb{E}\big[\sum_{t=0}^\infty \gamma^t \tilde{r}(s_t, a_t) \big| s_0 = s', a_0 = a'\big]
    \geq & \sum_{t = 0}^\infty \gamma^t (R_{\textrm{min}} - \epsilon)
    = & \frac{R_{\textrm{min}} - \epsilon}{1 - \gamma}
    \end{aligned}
\end{equation}
where the lower bound of $Q^\pi(s', a')$ is achieved if $(s', a') \notin \mathcal{S}_{\textrm{rev}}$. For $(s, a) \in \mathcal{S}_{\textrm{rev}}$,
\begin{align*}
    Q^\pi(s, a) &= \tilde{r}(s, a) + \gamma \mathbb{E}_{a' \sim \pi(\cdot \mid s') } \left[Q^\pi(s', a') \right]\\
    &\geq \tilde{r}(s, a) + \gamma \frac{R_{\textrm{min}} - \epsilon}{1 - \gamma}\\ 
    & \geq R_{\textrm{min}} + \gamma \frac{R_{\textrm{min}} - \epsilon}{1 - \gamma} \qquad \big[\tilde{r}(s, a) = r(s, a) \geq R_{\textrm{min}}, \text{ since } (s, a) \in \mathcal{S}_{\textrm{rev}} \big]  \\
    &= \epsilon + \left(R_{\textrm{min}} - \epsilon\right) + \gamma \frac{R_{\textrm{min}} - \epsilon}{1 - \gamma} = \epsilon + \frac{R_{\textrm{min}} - \epsilon}{1 - \gamma} = \epsilon + Q^\pi(s\_, a\_)
\end{align*}
where $(s\_, a\_) \not\in \mathcal{S}_{\textrm{rev}}$, implying $Q^\pi(s\_, a\_) = \frac{R_{\textrm{min}} - \epsilon}{1 - \gamma}$. This concludes the proof as ${Q^\pi(s, a) > Q^\pi(s\_, a\_)}$ for $\epsilon > 0$.
\end{proof}

To extend the discussion to stochastic dynamics, we redefine the reversibility set as ${\mathcal{S}_{\textrm{rev}} = \{(s, a) \mid \mathbb{P}\left(\mathcal{R}_\rho(s') = 1\right) \geq \eta_1\}}$, i.e, state-action pairs leading to a reversible state with at least $\eta_1$ probability and let ${\mathcal{S}_{\textrm{irrev}} = \{(s, a) \mid \mathbb{P}\left(\mathcal{R}_\rho(s') = 1\right) \leq \eta_2\}}$ denote the set of state-action pairs leading to irreversible states with at most $\eta_2$ probability. The goal is to show that actions leading to reversible states with high probability will have higher $Q$-values than those leading to irreversible states with high probability. The surrogate reward function $\tilde{r}$ is defined as:
\begin{equation} \label{eq:stoch_surrogate_reward_function}
    \Tilde{r}(s, a, s') = \begin{cases}
        r(s, a, s'), & \quad\mathcal{R}_\rho(s') = 1\\
        R_{\textrm{min}} - \epsilon,  & \quad\mathcal{R}_{\rho}(s') = 0
    \end{cases}
\end{equation}

\begin{theorem} \textit{Let $(s, a) \in \mathcal{S}_{\textrm{rev}}$ denote a state-action leading to a reversible state with at least $\eta_1$ probability and let ${(s\_, a\_) \in \mathcal{S}_{\textrm{irrev}}}$ denote a state-action pair leading to an irreversible state with at least $(1-\eta_2)$ probability. Assuming $\eta_1 > \eta_2 / (1-\gamma)$,
$$Q^\pi(s, a) > Q^\pi(s\_, a\_)$$
for all $\epsilon > \frac{\eta_2}{\eta_1 - \gamma \eta_1 -\eta_2} \left(R_{\textrm{max}} - R_{\textrm{min}}\right)$ and for all policies $\pi$.}
\end{theorem}

\begin{proof}
For any ${(s, a) \in \mathcal{S}_{\textrm{rev}}}$, we have the following (using equation \ref{eq:stoch_surrogate_reward_function}):
\begin{align}
    Q^\pi(s, a) &= \mathbb{P}\left(\mathcal{R}_\rho(s') = 1\right) \Big(r(s, a) + \gamma\mathbb{E}_{a' \sim \pi(\cdot \mid s')}\left[ Q^\pi(s', a')\right]\Big) +\mathbb{P}\left(\mathcal{R}_\rho(s') = 0\right)\frac{R_{\textrm{min}} - \epsilon}{1 - \gamma}\nonumber\\
    &\geq \mathbb{P}\left(\mathcal{R}_\rho(s') = 1\right) \Big(R_{\textrm{min}} + \gamma\frac{R_{\textrm{min}} - \epsilon}{1 - \gamma}\Big) +\mathbb{P}\left(\mathcal{R}_\rho(s') = 0\right)\frac{R_{\textrm{min}} - \epsilon}{1 - \gamma}\nonumber\\
    &\geq \eta_1 \Big(R_{\textrm{min}} + \gamma\frac{R_{\textrm{min}} - \epsilon}{1 - \gamma}\Big) + \left(1 - \eta_1 \right)\frac{R_{\textrm{min}} - \epsilon}{1 - \gamma}\nonumber\\
    &= \eta_1 R_{\textrm{min}} + (1 + \gamma \eta_1 - \eta_1) \frac{R_{\textrm{min}}-\epsilon}{1-\gamma}
\end{align}
where $\mathbb{P}(\mathcal{R}_\rho(s') = 1) \geq \eta_1$. 
% Similarly, for any $(s, a) \not\in \mathcal{S}_{\textrm{rev}}$, we have:
% \begin{align*}
%     \Tilde{r}(s, a) &= \mathbb{P}\left(\mathcal{R}_\rho(s') = 1\right) r(s, a) + \left(1-\mathbb{P}(\mathcal{R}_\rho(s') = 1)\right) R_{\textrm{min}}-\epsilon\\
%     &\leq \mathbb{P}\left(\mathcal{R}_\rho(s') = 1\right) R_{\textrm{max}} + \left(1-\mathbb{P}(\mathcal{R}_\rho(s') = 1)\right) R_{\textrm{min}}-\epsilon\\
%     &< \eta R_{\textrm{max}} + (1-\eta) \left(R_{\textrm{min}}-\epsilon\right)
% \end{align*}
For any $(s, a) \in \mathcal{S}_{\textrm{irrev}}$,
\begin{align}
    Q^\pi(s, a) &= \mathbb{P}\left(\mathcal{R}_\rho(s') = 1\right) \Big(r(s, a) + \gamma\mathbb{E}_{a' \sim \pi(\cdot \mid s')}\left[ Q^\pi(s', a')\right]\Big) +\mathbb{P}\left(\mathcal{R}_\rho(s') = 0\right)\frac{R_{\textrm{min}} - \epsilon}{1 - \gamma}\nonumber\\
    &\leq \mathbb{P}\left(\mathcal{R}_\rho(s') = 1\right) \frac{R_{\textrm{max}}}{1-\gamma} +\mathbb{P}\left(\mathcal{R}_\rho(s') = 0\right)\frac{R_{\textrm{min}} - \epsilon}{1 - \gamma}\nonumber\\
    &\leq \eta_2 \frac{R_{\textrm{max}}}{1-\gamma} + (1-\eta_2)\frac{R_{\textrm{min}} - \epsilon}{1 - \gamma}
\end{align}
as $\mathbb{P}(\mathcal{R}_\rho(s') = 1) \leq \eta_2)$ by definition of $\mathcal{S}_{\textrm{irrev}}$. Under the assumption that $\eta_1 > \eta_2 / (1-\gamma)$, whenever
\begin{align*}
    \eta_1 R_{\textrm{min}} + (1 + \gamma \eta_1 - \eta_1) \frac{R_{\textrm{min}}-\epsilon}{1-\gamma} &> \eta_2 \frac{R_{\textrm{max}}}{1-\gamma} + (1-\eta_2)\frac{R_{\textrm{min}} - \epsilon}{1 - \gamma}\\
    \eta_1 R_{\textrm{min}} - \eta_2 \frac{R_{\textrm{max}}}{1-\gamma} &> (\eta_1 - \gamma \eta_1 - \eta_2) \frac{R_{\textrm{min}}-\epsilon}{1-\gamma} \\
    \eta_1 (1-\gamma) R_{\textrm{min}} - \eta_2 R_{\textrm{max}} &> (\eta_1 - \gamma \eta_1 - \eta_2) (R_{\textrm{min}}-\epsilon)\\
    \epsilon &> \frac{\eta_2}{(\eta_1 - \gamma\eta_1 -\eta_2)} \left(R_{\textrm{max}} - R_{\textrm{min}}\right),
\end{align*}

ensures that $Q^\pi(s, a) > Q^\pi(s\_, a\_)$ for all $(s, a) \in \mathcal{S}_{\textrm{rev}}$ and for all $(s\_, a\_) \in \mathcal{S}_{\textrm{irrev}}$, finishing the proof.
\end{proof}

The above proof guarantees that for actions that lead to reversible states with high probability will have higher $Q$-values than actions leading to irreversible states with high probabilities for \textit{all policies}, even under stochastic dynamics. The restrictive assumption of $\eta_1 > \eta_2 / (1-\gamma)$ is required to ensure that the $Q$-value for worst $(s, a) \in \mathcal{S}_{\textrm{rev}}$ is better than the best $(s, a) \in \mathcal{S}_{\textrm{irrev}}$. An alternate analysis can be found in \citep[Appendix A.2]{thomas2021safe}, where the guarantees are only given for the optimal $Q$-functions but under less stringent assumptions. Improved guarantees are deferred to future work.

\subsubsection{On Empirical Bellman Backup Operator}
\label{app:emp_bell}

The empirical Bellman backup operator was introduced in subsection \ref{subsection:empirical_reversability}. In this section, we prove that it is a contraction, and analyze the convergence under empirical Bellman backup. For ${\hat{\mathcal{R}}_\rho: \mathcal{S} \mapsto [0, 1]}$, the empirical Bellman backup operator can be written as:

\begin{align}
    \hat{\mathcal{B}}^\pi Q(s, a) = \mathbb{E}_{s'\sim \mathcal{P}(\cdot \mid s, a)}\left[\hat{\mathcal{R}}_{\rho}(s')\left(r(s, a) + \gamma \mathbb{E}_{a' \sim \pi(\cdot \mid s')}Q(s', a')\right) + \left(1-\hat{\mathcal{R}}_\rho(s')\right) \frac{R_{\textrm{min}} - \epsilon}{1-\gamma}\right]\label{eq:bellman_empricial_unsafe}
\end{align}

\begin{theorem} \label{thm:empirical_bellman_contraction}
Empirical Bellman backup operator in equation~\ref{eq:bellman_empricial_unsafe} is a contraction under the $L^\infty$ norm.
\end{theorem}

\begin{proof}
For all $(s, a) \in \mathcal{S} \times \mathcal{A}$, we have the following:
\begin{equation*}
    \begin{aligned}
    \left|\hat{\mathcal{B}}^\pi Q(s, a) - \hat{\mathcal{B}}^\pi Q'(s, a)\right| &= \gamma \left| \mathbb{E}_{s'\sim \mathcal{P}(\cdot \mid s, a)} \left[ \hat{\mathcal{R}}_\rho(s')\mathbb{E}_{a' \sim \pi(\cdot \mid s')}\left[Q(s', a') - Q'(s', a')\right]\right] \right| \\
    &\leq  \gamma \left| \mathbb{E}_{s'\sim \mathcal{P}(\cdot \mid s, a)} \left[ \mathbb{E}_{a' \sim \pi(\cdot \mid s')}\left[Q(s', a') - Q'(s', a')\right]\right] \right| 
    \\ & \hspace{1.5cm} \left(\text{since } \hat{\mathcal{R}}_\rho(s') \in [0, 1]\right)\\
    &\leq \gamma \max_{(s'', a'') \in \mathcal{S} \times \mathcal{A}} \left|Q(s'', a'') - Q'(s'', a'') \right| \\
    &= \gamma \lVert Q - Q'\rVert_\infty
    \end{aligned}
\end{equation*}

Since this holds for all $(s, a) \in \mathcal{S} \times \mathcal{A}$, this implies:
$$||\hat{\mathcal{B}}^\pi Q - \hat{\mathcal{B}}^\pi Q'||_\infty = \max_{(s, a) \in \mathcal{S} \times \mathcal{A}} |\hat{\mathcal{B}}^\pi Q(s, a) - \hat{\mathcal{B}}^\pi Q'(s, a)| \leq \gamma ||Q - Q'||_\infty$$

The discount factor $\gamma < 1$, proving our claim.
\end{proof}

For a policy $\pi$, let $Q^\pi$ be the true $Q$-value function computed using conventional Bellman backup $\mathcal{B}^\pi$ and $\hat{Q}^\pi$ be the $Q$-values computed using the empirical Bellman backup $\hat{\mathcal{B}}^\pi$. Being fixed point of the operators, we have $Q^\pi = \mathcal{B}^\pi Q^\pi$ and $\hat{Q}^\pi = \hat{\mathcal{B}}^\pi \hat{Q}^\pi$. The following theorem relates the two:
\begin{lemma} \label{thm:suboptimality_bound}
Assuming that $\lVert \mathcal{R}_{\rho} - \hat{\mathcal{R}}_{\rho} \rVert_\infty \leq \delta$, the difference between true $Q$-values and empirical $Q$-values for any policy $\pi$ obeys the following inequality:
\begin{equation*}
    \left|Q^\pi(s, a) - \hat{Q}^\pi(s, a)\right| \leq \frac{\delta \left(R_{\textrm{max}} - R_{\textrm{min}} + \epsilon\right)}{(1-\gamma)^2}
\end{equation*}
\end{lemma}
\begin{proof}

Since $Q^\pi$ is the fixed point of $\mathcal{B}^\pi$ and $\hat{Q}^\pi$ is the fixed point of $\hat{\mathcal{B}}^\pi$, we can write:
\begin{align}
    \left|Q^\pi(s, a) - \hat{Q}^\pi(s, a)\right| &= \left|\mathcal{B}^\pi Q^\pi - \hat{\mathcal{B}}^\pi \hat{Q}^\pi \right|\nonumber\\
    &=\Bigg|\mathbb{E}_{s'\sim \mathcal{P}(\cdot \mid s, a)}\bigg[\left(\mathcal{R}_\rho(s') - \hat{\mathcal{R}}_{\rho}(s')\right)\left(r(s, a) - \frac{R_{\textrm{min}} - \epsilon}{1-\gamma}\right) \nonumber\\
    & \qquad\qquad\qquad + \gamma \mathbb{E}_{a' \sim \pi(\cdot \mid s)} \left[\mathcal{R}_\rho(s')Q^\pi(s', a') - \hat{\mathcal{R}}_\rho(s')\hat{Q}^\pi(s', a') \right]\bigg]\Bigg|\label{eq:big_eq}
\end{align}

Consider the following identity:
\begin{align}
    \mathcal{R}_\rho(s')Q^\pi(s', a') - \hat{\mathcal{R}}_\rho(s')\hat{Q}^\pi(s', a') = \mathcal{R}_\rho(s')Q^\pi(s', a') - &\mathcal{R}_\rho(s')\hat{Q}^\pi(s', a') \nonumber\\
    +&\mathcal{R}_\rho(s')\hat{Q}^\pi(s', a') - \hat{\mathcal{R}}_\rho(s')\hat{Q}^\pi(s', a')\nonumber\\
    = \mathcal{R}_\rho(s') \Big( Q^\pi(s', a') - &\hat{Q}^\pi(s', a') \Big) \nonumber\\
    +&\Big(\mathcal{R}_\rho(s') - \hat{\mathcal{R}}_\rho(s')\Big)\hat{Q}(s', a')\label{eq:simplification}
\end{align}
Plugging Eq~\ref{eq:simplification} in Eq~\ref{eq:big_eq}, we get:
\begin{align*}
   \left|Q^\pi(s, a) - \hat{Q}^\pi(s, a)\right| &= \Bigg|\mathbb{E}_{s'\sim \mathcal{P}(\cdot \mid s, a)}\bigg[\left(\mathcal{R}_\rho(s') - \hat{\mathcal{R}}_{\rho}(s')\right)\left(r(s, a) + \gamma \mathbb{E}_{a'\sim \pi(\cdot \mid s')}[\hat{Q}^\pi(s', a')] - \frac{R_{\textrm{min}} - \epsilon}{1-\gamma}\right) \nonumber\\
    & \qquad\qquad\qquad\qquad + \gamma \mathcal{R}_\rho(s') \mathbb{E}_{a' \sim \pi(\cdot \mid s)} \left[Q^\pi(s', a') - \hat{Q}^\pi(s', a')\right]\bigg]\Bigg|
\end{align*}
Taking the modulus inside the expectation and using the triangle inequality, we get:
\begin{align*}
\left|Q^\pi(s, a) - \hat{Q}^\pi(s, a)\right| \leq \mathbb{E}_{s'\sim \mathcal{P}(\cdot \mid s, a)}\bigg[\left|\mathcal{R}_\rho(s') - \hat{\mathcal{R}}_{\rho}(s')\right|\left|r(s, a) + \gamma \mathbb{E}_{a'\sim \pi(\cdot \mid s')}[\hat{Q}^\pi(s', a')] - \frac{R_{\textrm{min}} - \epsilon}{1-\gamma}\right| \nonumber\\
    \qquad\qquad\qquad\qquad + \gamma \left|\mathcal{R}_\rho(s')\right| \mathbb{E}_{a' \sim \pi(\cdot \mid s)} \left|Q^\pi(s', a') - \hat{Q}^\pi(s', a')\right|\bigg]\nonumber\\
\end{align*}
Using the following inequalities: $r(s, a) + \gamma \mathbb{E}_{a'\sim \pi(\cdot \mid s')}[\hat{Q}^\pi(s', a')] \leq R_{\textrm{max}} / (1-\gamma)$ as the maximum environment reward is $R_{\textrm{max}}$, $\left|\hat{\mathcal{R}}_\rho(s')\right| \leq 1$ as $\hat{\mathcal{R}}_\rho \in [0,1]$ and the assumption ${\lVert \mathcal{R}_{\rho} - \hat{\mathcal{R}}_{\rho} \rVert_\infty \leq \delta}$, we can complete the proof:
\begin{align*}
\left|Q^\pi(s, a) - \hat{Q}^\pi(s, a)\right| &\leq \delta\frac{R_{\textrm{max}} - R_{\textrm{min}} + \epsilon}{1-\gamma} + \gamma \mathbb{E}_{s'\sim \mathcal{P}(\cdot \mid s, a), a' \sim \pi(\cdot \mid s)} \left|Q^\pi(s', a') - \hat{Q}^\pi(s', a')\right|\nonumber\\
&\leq \delta\frac{R_{\textrm{max}} - R_{\textrm{min}} + \epsilon}{1-\gamma} + \gamma \mathbb{E}_{s'\sim \mathcal{P}(\cdot \mid s, a), a' \sim \pi(\cdot \mid s)} \left[\delta\frac{R_{\textrm{max}} - R_{\textrm{min}} + \epsilon}{1-\gamma} + \gamma \mathbb{E} \ldots \right]\\
&\leq \delta\frac{R_{\textrm{max}} - R_{\textrm{min}} + \epsilon}{(1-\gamma)^2}
\end{align*}
\end{proof}

Lemma~\ref{thm:suboptimality_bound} gives us a bound on the difference between the true $Q$-values and $Q$-values computed using the reversibility estimator $\hat{\mathcal{R}}_\rho$ for any policy $\pi$. The bound and the proof also suggest that closer $\hat{\mathcal{R}}_\rho(s)$ is to $\mathcal{R}_\rho$, closer the estimated $Q$-values are to true ones.

Similar to the empirical Bellman backup, we can define the empirical Bellman optimality operator:
\begin{align}
    \hat{\mathcal{B}}^* Q(s, a) = \mathbb{E}_{s'\sim \mathcal{P}(\cdot \mid s, a)}\left[\hat{\mathcal{R}}_{\rho}(s')\left(r(s, a) + \gamma \max_{a'} Q(s', a')\right) + \left(1-\hat{\mathcal{R}}_\rho(s')\right) \frac{R_{\textrm{min}} - \epsilon}{1-\gamma}\right]\label{eq:bellman_optimal_unsafe}
\end{align}
The empirical Bellman optimality operator is also a contraction, following a proof similar as that for empirical Bellman backup. Let $\hat{Q}^*$ denote the fixed point of Bellman optimality operator, that is $\hat{B}^* \hat{Q}^* = \hat{Q}^*$, and let $\hat{\pi}^*$  denote the greedy policy with respect to $\hat{Q}^*$. As a final result,

\textbf{Theorem 5.1.}
\textit{Let $\pi^*$ denote the optimal policy and $Q^*$ denote the corresponding optimal $Q$-value function. Let $\hat{\pi}^*$ denote the optimal policy returned by empirical Bellman optimality operator $\hat{\mathcal{B}}^*$. Assuming $\lVert \mathcal{R}_{\rho} - \hat{\mathcal{R}}_{\rho} \rVert_\infty \leq \delta$,}
\begin{equation*}
    Q^{\hat{\pi}^*}(s, a) \geq Q^*(s, a) - \frac{2\delta \left(R_{\textrm{max}} - R_{\textrm{min}} + \epsilon\right)}{(1-\gamma)^2}
\end{equation*}
\textit{for all $(s, a) \in \mathcal{S} \times \mathcal{A}$.}

\begin{proof}
For clarity of notation, we will use $Q(\pi)$ denote $Q^\pi(s, a)$ and $\hat{Q}(\pi)$ denote $\hat{Q}^\pi(s, a)$ for a policy $\pi$. Now,
\begin{align*}
    Q^*(s, a) - Q^{\hat{\pi}^*}(s, a) = \left(Q(\pi^*) - \hat{Q}(\pi^*)\right) + \left(\hat{Q}(\pi^*) - \hat{Q}(\hat{\pi}^*)\right) + \left(\hat{Q}(\hat{\pi}^*) - Q(\hat{\pi}^*)\right)
\end{align*}

Using the fact $\hat{\pi}^*$ is the optimal policy with respect to $\hat{\mathcal{B}}^*$, we have ${\hat{Q}(\pi^*) - \hat{Q}(\hat{\pi}^*)} \leq 0$. This implies that:
\begin{align}
    Q^*(s, a) - Q^{\hat{\pi}^*}(s, a) \leq \left(Q(\pi^*) - \hat{Q}(\pi^*)\right)  + \left(\hat{Q}(\hat{\pi}^*) - Q(\hat{\pi}^*)\right)\label{eq:q_decomp}
\end{align}

Using lemma~\ref{thm:suboptimality_bound}, we have
\begin{align*}
    \left(Q(\pi^*) - \hat{Q}(\pi^*)\right) &\leq \frac{\delta \left(R_{\textrm{max}} - R_{\textrm{min}} + \epsilon\right)}{(1-\gamma)^2}\\
    \left(\hat{Q}(\hat{\pi}^*) - Q(\hat{\pi}^*)\right) &\leq \frac{\delta \left(R_{\textrm{max}} - R_{\textrm{min}} + \epsilon\right)}{(1-\gamma)^2}
\end{align*}
Plugging these in Eq~\ref{eq:q_decomp}, we get
\begin{align*}
    Q^*(s, a) - Q^{\hat{\pi}^*}(s, a) \leq \frac{2\delta \left(R_{\textrm{max}} - R_{\textrm{min}} + \epsilon\right)}{(1-\gamma)^2}
\end{align*}

Rearranging the above bound gives us the statement in the theorem.
\end{proof}

Theorem~\ref{thm:final} gives us the assurance that as long as the estimator $\hat{\mathcal{R}}_\rho$ is close to $\mathcal{R}_\rho$, the $Q$-values of the greedy policy obtained by value iteration using $\hat{\mathcal{B}}^*$, i.e. $\hat{\pi}^*$ will not be much worse than the optimal $Q$-values. The above result also suggests choosing a smaller $\epsilon$ for a smaller gap in performance.

\subsection{Detailed Pseudocode}
\label{app:detailedpseudocode}

In Algorithms~\ref{alg:episodic_paint} and~\ref{alg:nonepisodic_paint}, we provide pseudo-code for the episodic and non-episodic variants of our method PAINT. We build the non-episodic variant upon the MEDAL algorithm~\cite{sharma2022state}, which introduces a backward policy whose objective is to match the state distribution of the forward demonstrations (summarized in Section~\ref{sec:prelims}).

% In the episodic variant, PAINT aborts the learning agent early in an episode if the reversibility classifier $\hat{\mathcal{R}}_\rho$ indicates that the agent has entered an irreversible state. Then, for the remainder of the episode, the agent takes uniform-random actions. At the end of each episode of length $H$, PAINT queries the reversibility labels for all $H$ states via Algorithm~\ref{alg:bin_search}, and trains the reversibility classifier $\hat{\mathcal{R}}_\rho$ on all of the collected labels thus far.

\begin{algorithm}
    \SetAlgoLined
    \textbf{optional:} forward demonstrations $\mathcal{N}$ \\
    \textbf{initialize:} $\pi, Q, \mathcal{D}$; {\color{olive} // forward agent parameters} \\
    \textbf{initialize} $\hat{\mathcal{R}}_\rho, \mathcal{D}_\rho$; {\color{olive} // reversibility classifier and dataset of labels} \\
    {\color{olive} // add demonstrations to replay buffer}\\
    $\mathcal{D} \gets \mathcal{D} \cup \mathcal{N}$\; 
    \While{not done}{
        $s \sim \rho_0$ {\color{olive}// reset environment} \\
        $\mathcal{D}_{\textrm{new}} \gets \mathcal{D}_{\textrm{new}} \cup \{s\}$\;
        aborted $\leftarrow$ False; \\
        \For{$t = 1, 2, \dots, H$}{
            \If{not aborted}{
                $a \sim \pi( \cdot \mid s)$; \\
                update $\pi, Q$; {\color{brown}// Eq~\ref{eq:bellman_empricial_unsafe}}
            }
            \Else{
                $a \sim \textrm{unif}(\mathcal{A})$;
            }
            $s' \sim \mathcal{P}(\cdot \mid s, a), r \leftarrow r(s, a)$; \\
            $\mathcal{D} \gets \mathcal{D} \cup \{(s, a, s', r)\}$; \\
            \If{not aborted and $\hat{R}_\rho(s) \le 0.5$}{
                aborted $\leftarrow$ True;
            }
            $\mathcal{D}_{\textrm{new}} \gets \mathcal{D}_{\textrm{new}} \cup \{s'\}$\;
            $s \leftarrow s'$;
        }
        {\color{olive}// query reversibility labels for newly collected states via Alg.~\ref{alg:bin_search}} \\
        label $\mathcal{D}_{\textrm{new}}$\;
        $\mathcal{D}_\rho \leftarrow \mathcal{D}_\rho \cup \mathcal{D}_{\textrm{new}}$\;
        $\mathcal{D}_{\textrm{new}} \gets \emptyset$\;
        {\color{olive}// train classifier on all labeled data, new and old} \\
        update $\hat{\mathcal{R}}_\rho$;
    }
    \caption{PAINT (Episodic)}
    \label{alg:episodic_paint}
\end{algorithm}

\begin{algorithm}
    \SetAlgoLined
    \textbf{input:} forward demonstrations $\mathcal{N}_f$\;
    \textbf{optional:} backward demonstrations $\mathcal{N}_b$\;
    \textbf{initialize:} $\pi_f, Q^f, \mathcal{D}_f$; {\color{olive} // forward agent parameters}\\
    \textbf{initialize:} $\pi_b, Q^b, \mathcal{D}_b$; {\color{olive} // backward agent parameters}\\
    \textbf{initialize} $\hat{\mathcal{R}}_\rho, \mathcal{D}_\rho$; {\color{olive} // reversibility classifier and dataset of labels} \\
    \textbf{initialize} $C(s)$; {\color{olive} // state-space discriminator for backward policy} \\
    {\color{olive} // add demonstrations to replay buffer}\\
    $\mathcal{D}_f \gets \mathcal{D}_f \cup \mathcal{N}_f$\; 
    $\mathcal{D}_b \gets \mathcal{D}_b \cup \mathcal{N}_b$\;
    $\mathcal{D}_{\textrm{new}} \gets \emptyset$ \;
    \While{not done}{
        $s \sim \rho_0$; {\color{olive}// reset environment} \\
        $\mathcal{D}_{\textrm{new}} \gets \mathcal{D}_{\textrm{new}} \cup \{s\}$\;
        {\color{olive}// continue till the reversibility classifier detects an irreversible state} \\
        \While{$\hat{\mathcal{R}}_\rho(s) > 0.5$}{
            {\color{olive}// run forward policy for a fixed number of steps, switch to backward policy} \\
            \If{forward}{
                $a \sim \pi_f( \cdot \mid s)$;\\
                $s' \sim \mathcal{P}(\cdot \mid s, a), r \leftarrow r(s, a)$;\\
                $\mathcal{D}_f \gets \mathcal{D}_f \cup \{(s, a, s', r\}$; \\
                update $\pi_f, Q^{f}$; {\color{brown}// Eq~\ref{eq:bellman_empricial_unsafe}} \\
            }
            \Else{
                $a \sim \pi_b( \cdot \mid s)$;\\
                $s' \sim \mathcal{P}(\cdot \mid s, a), r \leftarrow -\log (1 - C(s'))$;\\
                $\mathcal{D}_b \gets \mathcal{D}_b \cup \{(s, a, s', r)\}$; \\
                update $\pi_b, Q^{b}$; {\color{brown}// Eq~\ref{eq:bellman_empricial_unsafe}} \\
            }
            {\color{olive} // train disriminator every $K$ steps}\\
            \If{train-discriminator}{
                {\color{olive} // sample a batch of positives $S_p$ from the forward demos $\mathcal{N}_f$, and a batch of negatives $S_n$ from backward replay buffer $\mathcal{D}_b$}\\
                $S_p \sim \mathcal{N}_f, S_n \sim \mathcal{D}_b$;\\
                update $C$ on $S_p \cup S_n$;\\
            }
            $\mathcal{D}_{\textrm{new}} \gets \mathcal{D}_{\textrm{new}} \cup \{s'\}$\;
            $s \gets s'$;\\
        }
        {\color{olive}// optionally explore environment}\\
        \For{explore steps}{
            $a \sim \textrm{unif}(\mathcal{A})$, $s' \sim \mathcal{P}(\cdot \mid s, a), r \gets r(s, a)$\;
            update $\mathcal{D}_f, \mathcal{D}_b$; {\color{olive} // use $C(s)$ for the reward labels in $\mathcal{D}_b$}\\
            $\mathcal{D}_{\textrm{new}} \gets \mathcal{D}_{\textrm{new}} \cup \{s'\}$\;
            $s \gets s'$;\\
        }
        {\color{olive}// query reversibility labels for newly collected states via Alg~\ref{alg:bin_search}}\\
        label $\mathcal{D}_{\textrm{new}}$\;
        $\mathcal{D}_\rho \leftarrow \mathcal{D}_\rho \cup \mathcal{D}_{\textrm{new}}$\;
        $\mathcal{D}_{\textrm{new}} \gets \emptyset$\;
        {\color{olive}// train classifier on all labeled data, new and old}\\
        update $\hat{\mathcal{R}}_\rho$ on $\mathcal{D}_\rho$\;
    }
    \caption{PAINT with MEDAL~\citep{sharma2022state} (Non-episodic)
    }
    \label{alg:nonepisodic_paint}
\end{algorithm}

\subsection{Environment Details}
\label{app:envs}

In this section, we provide details of each of the four irreversible environments, which are visualized in Fig.~\ref{fig:envs} and Fig.~\ref{fig:maze_results}.

% objective
% what is irreversible behavior
% what reset behavior is
% how often resets are provided
% agent's state space, action space, and reward function
% extra information provided to the agent

\textbf{Maze}. In this $2$-D continuous-control environment, the agent is a point mass particle that starts in the top left corner of the maze and must reach the bottom right corner. Throughout the environment, there are trenches (marked in black) that the agent must avoid. Entering them is irreversible: the agent can roam freely within the trench but cannot leave it without an environment reset. The agent is placed back at the top left corner upon a reset, which is provided every $500$ time-steps. 
The agent's state space consists of its $xy$-position, and two control inputs correspond to the change applied to its $xy$-position. The reward function is defined as $r_t = \mathds{1}( \| s_t - g \|_2 < 0.1)$, where $g$ is the goal position. We provide $10$ demonstrations of the task to the agent at the beginning of training.

\textbf{Tabletop Organization}. This environment modifies the Tabletop Organization task from~\citet{sharma2021autonomous}. The agent's objective is to grasp the mug and move it to one of the four specified goal positions. Grasping the mug and dropping it off beyond the red boundary is irreversible, i.e., the agent can no longer re-grasp the mug. Upon a reset, which is provided every $200$K time-steps or when requested by the agent, the agent is placed back at the center of the table and the mug is placed on its right, just within the red boundary (see Fig.~\ref{fig:envs}). The agent's state space consists of the its own $xy$-position, the mug's $xy$-position, an indicator of whether the mug is grasped, the goal position of the mug, and finally, the goal position of the agent after putting the mug down at its goal. There are three control inputs, which apply changes to the agent's $xy$-position and toggle between grasping, if the object is nearby (i.e., within a distance of $0.4$), and releasing, if the object is currently grasped. The agent's reward function is $r_t = \mathds{1}(\| s_t - g \|_2 < 0.1)$, i.e., both the agent's $xy$-position and the mug's $xy$-position must be close to their targets. We provide $50$ forward demonstrations, $50$ backward demonstrations, and $1000$ examples of (randomly generated) irreversible states to the agent at the beginning of training.

\textbf{Peg Insertion}. This environment modifies the Peg Insertion task from~\citet{sharma2021autonomous}. The objective of this task is to grasp and insert the peg into the hole in the box. We modified the table so that the raised edges that stop the peg from rolling off the table are removed and the table is significantly narrower. Hence, the peg may fall off the table, which cannot be reversed by the robot. Instead, when the environment is reset, which automatically occurs every $100$K time-steps or when requested by the agent, the peg is placed back at one of $15$ possible initial positions on the table. The agent's state space consists of the robot's $xyz$-position, the distance between the robot's gripper fingers, and the object's $xyz$-position. The agent's action consists of 3D end-effector control and normalized gripper torque. Let $s^\text{peg}_t$ represent the state of the peg and $g^\text{peg}$ be its goal state, then the reward function is $r_t = \mathds{1}(\|s^\text{peg}_t - g^\text{peg} \|_2 < 0.05 )$. We provide the agent with $12$ forward demonstrations and $12$ backward demonstrations. 

\textbf{Half-Cheetah}. We design this environment based on the version from~\citet{brockman2016openai}. In particular, the agent must run at one of six target velocities $\{ 3, 4, 5, 6, 7, 8 \}$, specified to the agent. Every $500$ time-steps, the target velocity switches to a different value selected at random. The agent's actions, which correspond to torque control of the cheetah's six joints, are scaled up by a factor of $5$. When the agent is flipped onto its back (i.e., its normalized orientation is greater than $2 \pi / 3$), we label these states irreversible. When the agent is reset, which only occurs when requested, the agent is placed upright again at angle of $0$. The agent's observation space consists of the velocity of the agent’s center of mass, angular
velocity of each of its six joints, and the target velocity. Let $v$ be the velocity of the agent, then the reward function, which is normalized to be between $0$ and $1$, is $r_t = 0.95 * (8 - v) / 8 + 0.05 * (6 - \|a_t\|_2^2) / 6$. There are no demonstrations for this task.

\subsection{Implementation Details}
\label{app:implementation}

Below, we provide implementation details of our algorithm PAINT and the baselines. Every algorithm, including ours, has the following components.

\textit{Forward policy network}. The agent's forward policy is represented by an MLP with $2$ fully-connected layers of size $256$ in all experimental domains, trained with the Soft Actor-Critic (SAC)~\citep{haarnoja2018soft} algorithm.

\textit{Forward critic network}. The agent's forward critic is represented by an MLP with $2$ fully-connected layers of size $256$ in all experimental domains, trained with the Soft Actor-Critic (SAC)~\citep{haarnoja2018soft} algorithm.

\textit{Balanced batches}. In the Tabletop Organization and Peg Insertion tasks, the agent's forward policy and critic networks are trained with batches that consist of demonstration tuples and of tuples sampled from the agent's online replay buffer. We control the ratio $p$ of demonstration tuples to online tuples with a linearly decaying schedule,
$$
p_t = \begin{cases}
    \frac{(p_T - p_0)}{T} t + p_0 & t < T \\
    p_T & T \ge t.
\end{cases}
$$
In the Tabletop Organization and Peg Insertion tasks, $p_0 = 0.5$, $p_T = 0.1$, and $T = 500$K. We do not train with balanced batches in the Maze task, and simply populate the online replay buffer with the demonstrations at the beginning of training.

\textbf{Episodic RL (Soft Actor-Critic)}~\citep{haarnoja2018soft}.
% size of networks (policy, critic)
% balanced batches
% horizons in each environ
In addition to the policy and critic networks trained with balanced batches, the episodic RL comparison requests for resets every $H'$ time-steps in the Tabletop Organization ($H' = 2000$), Peg Insertion ($H' = 1000$), and Half-Cheetah ($H' = 2000$) tasks. 

\textbf{Safe Model-Based Policy Optimization (SMBPO)}~\citep{thomas2021safe}.
% size of networks (policy, critic)
% balanced batches, rmin
This comparison trains the forward critic with the modified Bellman update $\mathcal{B}^\pi Q(s, a)$ defined in Eqn.~\ref{eq:bellman_unsafe}, where $\epsilon = 0$ in the Maze and Half-Cheetah tasks and $\epsilon = -0.1$ in the Tabletop Organization and Peg Insertion tasks. 

\textbf{Safety Q-functions for RL (SQRL)}~\citep{srinivasan2020learning}.
% size of networks (policy, critic)
% balanced batches
% epsilon
This comparison trains an additional safety critic $Q^\pi_\text{safe}$, which estimates the future probability of entering an irreversible state, and the policy is updated with the following augmented objective
$$
J^\pi_\text{safe}(\nu) = J^\pi + \mathbb{E}_{s \sim \mathcal{D}, a \sim \pi(\cdot | s)} \left[ \nu \left(\epsilon_\text{safe}, Q^\pi_\text{safe}(s, a) \right) \right],
$$
where $\nu$ is the Lagrange multiplier for the safety constraint and is updated via dual gradient descent. We only evaluate this comparison in the Maze task, where $\epsilon_\text{safe} = 10$. 

\subsubsection{Forward-Backward Algorithms}

In the forward-backward setups, we train a backward policy and critic in addition to their forward counterparts. The details of the backward components that are shared across all methods are described below.

\textit{Backward policy network}. The agent's backward policy is represented by an MLP with $2$ fully-connected layers of size $256$ in all experimental domains, trained with the SAC algorithm.

\textit{Backward critic network}. The agent's backward critic is represented by an MLP with $2$ fully-connected layers of size $256$ in all experimental domains, trained with the SAC algorithm.

\textit{Backward balanced batches}. The agent's backward policy and critic networks are trained with batches that consist of demonstration tuples and of tuples sampled from the agent's online replay buffer. We control the ratio $p$ of demonstration tuples to online tuples with a linearly decaying schedule,
$$
p_t = \begin{cases}
    \frac{(p_T - p_0)}{T} t + p_0 & t < T \\
    p_T & T \ge t.
\end{cases}
$$
In the Tabletop Organization and Peg Insertion tasks, $p_0 = 0.5$, $p_T = 0.1$, and $T = 500$K. 

\textbf{Leave No Trace (LNT)}~\citep{eysenbach2017leave}.
% size of networks (forward/backward policy, forward/backward critic)
% balanced batches, termination condition
This comparison additionally trains backward policy and critic networks, whose reward function is the sparse indicator of whether the current state is within some threshold of the initial state. The thresholds are the same as those used for the forward reward functions defined in~\ref{app:envs}. 

\textit{Policy switching}. Leave No Trace switches from the forward to backward policy if the backward critic's $Q$-value is lower than $\epsilon_\text{LNT}$ or after $300$ time-steps, and switches from the backward to forward policy after $300$ time-steps. In Tabletop, $\epsilon_\text{LNT} = 0.1$, and in Peg Insertion, $\epsilon_\text{LNT} = 0.005$.

\textit{Termination condition}. Leave No Trace additionally requests a reset if, after $300$ time-steps, the backward policy fails to bring the environment within a distance of $0.1$ of the initial state.

\textbf{Matching Expert Distributions for Autonomous Learning (MEDAL)}~\citep{sharma2022state}.
% size of networks (forward/backward policy, forward/backward critic, MEDAL classifier)
% balanced batches
Like Leave No Trace, MEDAL trains a backward policy and critic. However, instead of returning to the initial state, the backward reward function is whether the current state matches the distribution of demonstration states, formally defined in Eqn.~\ref{eqn:medal}.

\textit{MEDAL classifier}. The classifier $C$ in Eqn.~\ref{eqn:medal} is represented by an MLP with $1$ FC layer of size $128$.

\textit{Policy switching}. The algorithm switches policies (i.e., from forward to backward and from backward to forward) after every $300$ time-steps.

\textbf{Reversibility-Aware Exploration (RAE)}~\citep{grinsztajn2021there}. RAE trains a self-supervised reversibility estimator, specifically to predict whether a state transition $(s, \tilde{s})$ is more likely than the reverse transition $(\tilde{s}, s)$. RAE generates data for the binary classifier with a windowed approach. For every state trajectory $(s_{t:t+w})$ of length $w$ collected by the agent, all state pairs $(s_i, s_j)$, where $i < j$, are labeled \emph{positive}, and all pairs $(s_j, s_i)$, where $i < j$, are labeled \emph{negative}. For all experimental tasks, we use a window size of $w = 10$ time-steps. With this estimator, the forward critic is trained with the modified Bellman update $\hat{\mathcal{B}}^\pi Q(s, a)$, where $\epsilon = 0$ in the Maze and Half-Cheetah tasks and $\epsilon = -0.1$ in the Tabletop Organization and Peg Insertion tasks.

\textit{Reversibility classifier}. The classifier $\hat{\mathcal{R}}_\rho$ is represented by an MLP with $1$ FC layer of size $128$.

\textit{Termination condition}. In Maze and Half-Cheetah, the termination condition is $\hat{\mathcal{R}}_\rho > 0.5$. In Tabletop Organization and Peg Insertion, the condition is $\hat{\mathcal{R}}_\rho > 0.8$.

\textit{Exploration}. We augment RAE with uniform-random exploration after the termination condition is met as proposed in our method. In the Maze environment, the agent takes uniform-random actions for the rest of the episode (of length $500$). In Tabletop Organization and Peg Insertion, $N_\text{explore} = 300$ time-steps. In Half-Cheetah, $N_\text{explore} = 500$ time-steps.

In the Tabletop Organization and Peg Insertion tasks, we train an additional backward policy and critic, whose reward functions are defined in terms of the MEDAL classifier. The backward critic is also trained with the modified $\hat{\mathcal{B}}^\pi Q(s, a)$, with the same hyperparameters as the forward critic.

\textit{MEDAL classifier}. The classifier $C$ in Eqn.~\ref{eqn:medal} is represented by an MLP with $1$ FC layer of size $128$.

\textit{Policy switching}. The algorithm switches policies (i.e., from forward to backward and from backward to forward) after every $300$ time-steps.

\textbf{PAINT (Ours)}.
% size of networks (forward/backward policy, forward/backward critic, MEDAL classifier, reversibility classifier)
% balanced batches, termination condition, rmin, num explore steps
PAINT trains a reversibility classifier $\hat{\mathcal{R}}_\rho$ and checks whether the current state is estimated to be irreversible. If it is estimated to be irreversible, the agent takes uniform-random actions for $H_\text{explore}$ time-steps and requests a reset afterward. The forward critic is also trained with the modified Bellman update $\hat{\mathcal{B}}^\pi Q(s, a)$, where $\epsilon = 0$ in the Maze and Half-Cheetah tasks and $\epsilon = -0.1$ in the Tabletop Organization and Peg Insertion tasks.

\textit{Reversibility classifier}. The classifier $\hat{\mathcal{R}}_\rho$ is represented by an MLP with $1$ FC layer of size $128$.

\textit{Termination condition}. In all tasks, the termination condition is $\hat{\mathcal{R}}_\rho > 0.5$.

\textit{Exploration}. In the Maze environment, the agent takes uniform-random actions for the rest of the episode (of length $500$). In Tabletop Organization and Peg Insertion, $N_\text{explore} = 300$ time-steps. In Half-Cheetah, $N_\text{explore} = 500$ time-steps.

In the Tabletop Organization and Peg Insertion tasks, we train an additional backward policy and critic. The details for the backward policy and critic are the same as in RAE. %, whose reward functions are defined in terms of the MEDAL classifier. The backward critic is also trained with the modified Bellman update $\hat{\mathcal{B}}^\pi Q(s, a)$, with the same hyperparameters as the forward critic.

% \textit{MEDAL classifier}. The classifier $C$ in Eqn.~\ref{eqn:medal} is represented by an MLP with $1$ FC layer of size $128$.

% \textit{Backward policy network}. The agent's backward policy is represented by an MLP with $2$ fully-connected layers of size $50$ in all experimental domains, trained with the SAC algorithm.

% \textit{Backward critic network}. The agent's backward critic is represented by an MLP with $2$ fully-connected layers of size $50$ in all experimental domains, trained with the SAC algorithm.

% \textit{Backward balanced batches}. The agent's backward policy and critic networks are trained with batches that consist of demonstration tuples and of tuples sampled from the agent's online replay buffer. We control the ratio $p$ of demonstration tuples to online tuples with a linearly decaying schedule,
% $$
% p_t = \begin{cases}
%     \frac{(p_T - p_0)}{p_T} t + p_0 & t < T \\
%     p_T & T \ge t.
% \end{cases}
% $$
% In the Tabletop Organization and Peg Insertion tasks, $p_0 = 0.5$, $p_T = 0.1$, and $T = 500$K. 

% \textit{Policy switching}. The algorithm switches policies (i.e., from forward to backward and from backward to forward) after every $300$ time-steps.

\subsubsection{Codebase}

We have publicly released our code at this GitHub repo: \url{https://github.com/tajwarfahim/proactive_interventions}. Our codebase builds on top of codebases from~\citet{yarats2021drqv2, yarats2021image} and~\citet{sharma2021autonomous}.

\subsection{Discussion on Comparisons}
\label{app:comparisons}
\begin{table}[h]
    \small
    \centering
    \setlength{\tabcolsep}{4pt}
    \begin{tabular}{lccc}
        \toprule
        Method & Forward-Backward & Requires Reversibility Labels? & Intervention Rule  \\
        \midrule
        Episodic RL (SAC)~\citep{haarnoja2018soft} & No & No & N/A \\
        SMBPO~\citep{thomas2021safe} & No & Yes & N/A \\
        SMBPO w. Oracle Term & No & Yes & Oracle \\
        SQRL~\citep{srinivasan2020learning} & No & Yes & N/A \\
        LNT~\citep{eysenbach2017leave} & Yes & No & $s \not\in \text{supp}(\rho_0)$ \\
        MEDAL~\citep{sharma2022state} & Yes & No & N/A \\
        RAE~\citep{grinsztajn2021there} & Yes & No & $\hat{\mathcal{R}}_\textsc{RAE}(s) > p$ \\
        PAINT (Ours) & Yes & Yes & $\hat{\mathcal{R}}_\rho(s) < 0.5$ \\
        \bottomrule
    \end{tabular}
    \vspace{0.2cm}
    \caption{\small Summary of assumptions for each method.}
    \label{tab:comparisons}
    \vspace{-0.4cm}
\end{table}
The algorithms we compare to make varying assumptions, which we summarize in Table~\ref{tab:comparisons}. In particular, SMBPO~\citep{thomas2021safe} and SQRL~\citep{srinivasan2020learning} are safe RL methods, which we adapt to the setting in this work. They use reversibility labels as safety labels, but lack a backward policy to reset the agent and an intervention rule. We therefore only study these comparisons in the episodic setting. In the continuing setup, we equip SMBPO with an oracle intervention rule, which calls for a reset when the agent enters an irreversible state. The next comparisons LNT~\citep{eysenbach2017leave} and MEDAL~\citep{sharma2022state} are explicitly designed for the autonomous RL setting: they train forward and backward policies, which take alternate turns controlling the agent. They however do not use reversibility labels. Instead, LNT requires a different assumption, checking whether the agent has returned back to the support of the initial state distribution, while MEDAL has no defined intervention rule. Finally, we adapt RAE~\citep{grinsztajn2021there}, which is not originally designed for autonomous RL, by introducing a backward policy and designing an intervention rule based on the RAE reversibility classifier. Notably, RAE trains its classifier with self-supervised labels, and so does not require any explicit reversibility labels provided by an expert.

\subsection{Additional Experimental Results}
\label{app:add_exps}
In this section, we present additional plots to accompany Section~\ref{subsection:main_results} and additional ablations to accompany Section~\ref{subsection:ablations}.

\subsubsection{Additional Plots}

\begin{figure}
    \centering
    \includegraphics[width=\linewidth]{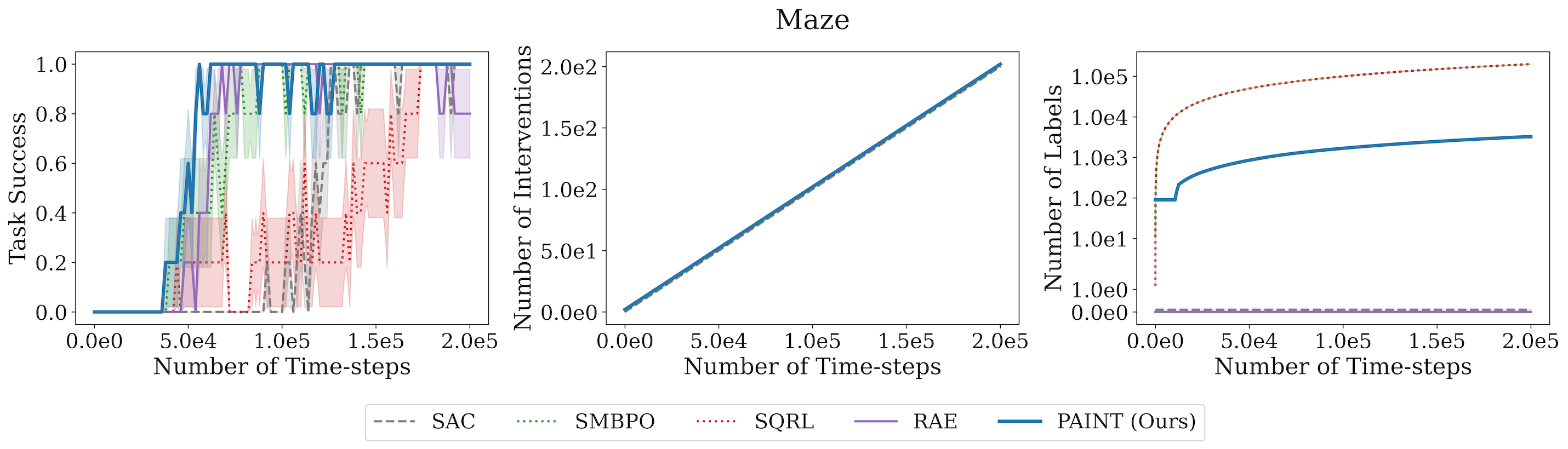} \\
    \vspace{0.3cm}
    \includegraphics[width=\linewidth]{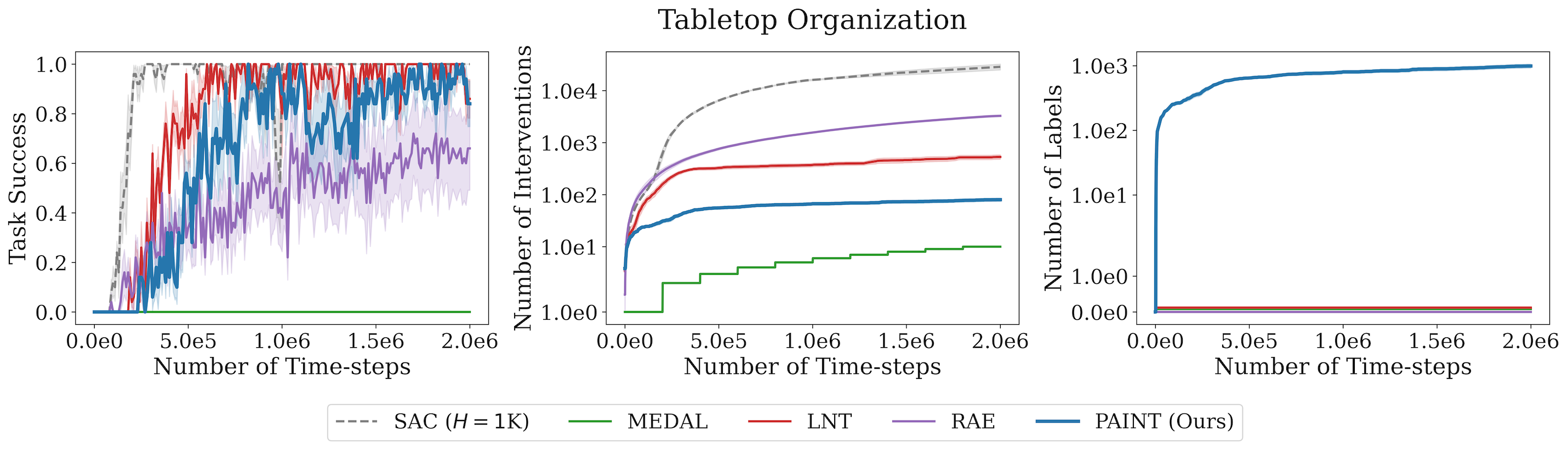} \\
    \vspace{0.3cm}
    \includegraphics[width=\linewidth]{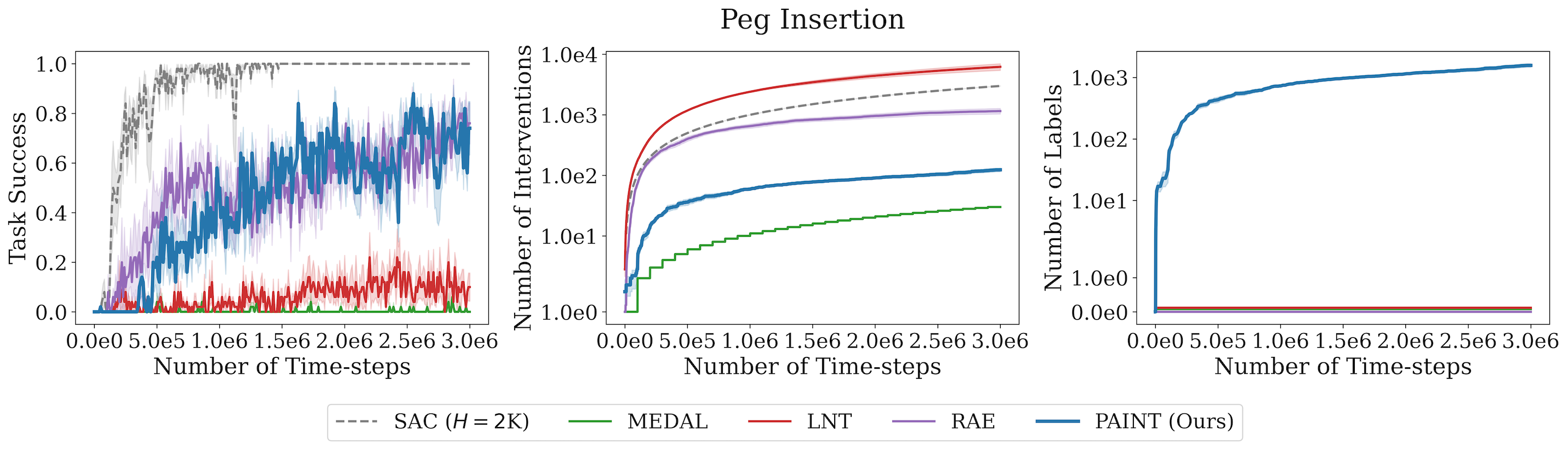} \\
    \vspace{0.3cm}
    \includegraphics[width=\linewidth]{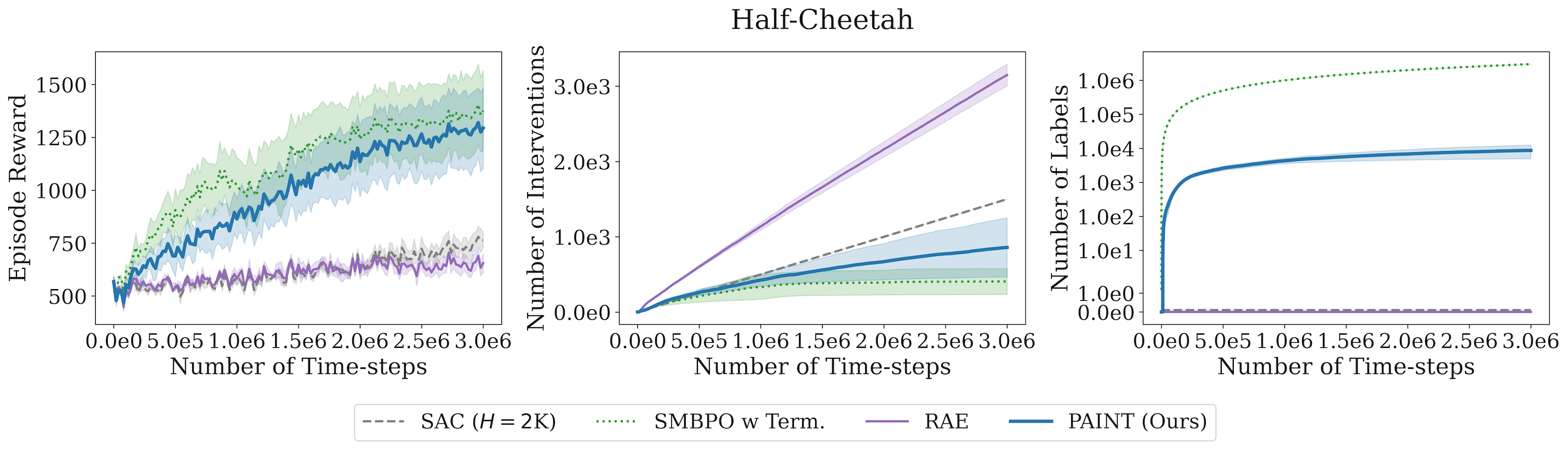}
    \caption{\small \textit{(left)} Task success versus time. \textit{(middle)} Number of interventions versus time. \textit{(right)} Number of queried labels versus time.}
    \label{fig:main_results_app}
\end{figure}

In Fig.~\ref{fig:main_results_app}, we plot the task success, number of reset interventions, and number of reversibility labels versus the number of time-steps in each of the four tasks.

\subsubsection{Ablations and Sensitivity Analysis}

\begin{figure*}
    \centering
    \includegraphics[width=.66\linewidth]{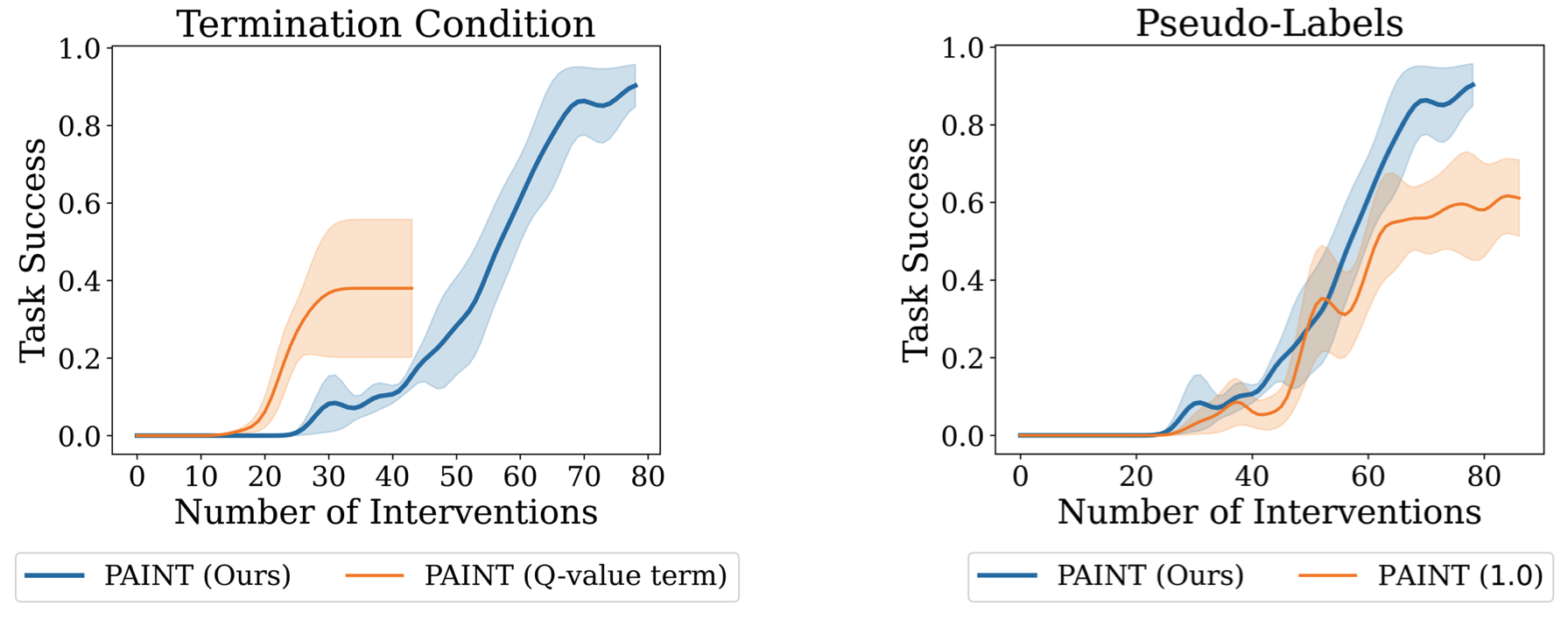}
    \caption{\small 
    (\textit{left}) We define a termination condition based on the $Q$-values, which learns with fewer resets but also achieves significantly lower final performance. (\textit{right}) For states collected during a trial, their reversibility labels are still unknown. When training the agent with Eqn.~\ref{eq:bellman_empricial_unsafe}, we use predictions from the reversibility classifier as pseudo-labels. We compare to labeling the unlabeled states one, i.e., treating them as reversible.}
    \label{fig:ablations_app}
\end{figure*}

\textbf{Termination conditions}. To evaluate the importance of the reversibility classifier, we study an alternative choice for the termination condition, one based on the $Q$-value function. Intuitively, the values for irreversible states will be low, as there is no path that leads to the goal from these states. We define the value-based termination condition as $V^\pi(s) < \epsilon$ where $\epsilon$ is the threshold. 
% % Alternatively, we can compare the value at the current state to the value of the initial state, i.e., $V^\pi(s_0) - V^\pi(s) > \eta$, where $\eta$ is the threshold. For both, 
We approximate the value $V^\pi(s) = \mathbb{E}_\pi [Q^\pi(s, a)]$ with the trained $Q$-value function evaluated at $N=10$ policy actions. We plot the task success in the Tabletop Manipulation task in Fig.~\ref{fig:ablations_app} (left). After $1$M time-steps, the $Q$-value-based termination condition requires fewer interventions, approximately half of the interventions needed by PAINT, but does not converge to the same final performance as PAINT with the reversibility classifier. 
Critically, PAINT with $Q$-value termination also trains a reversibility classifier to generate pseudo-labels for unlabeled states.

\textbf{Pseudo-labels for unlabeled states}. Currently, when updating the agent with Eqn.~\ref{eq:bellman_empricial_unsafe}, our method uses the predictions from the stuck classifier as pseudo-labels for the unlabeled states collected during a trial. An alternative choice is labeling them one, i.e., treating them as reversible. We evaluate this choice in the Tabletop Manipulation task, and plot the task success in Fig.~\ref{fig:ablations_app} (right). After $1$M time-steps, the number of reset interventions requested by both methods are similar. However, our method succeeds at the task almost $100\%$ of the time, while the agent trained with pseudo-labels of one only achieves success of around $60\%$.

\begin{figure}
    \centering
    \includegraphics[width=\linewidth]{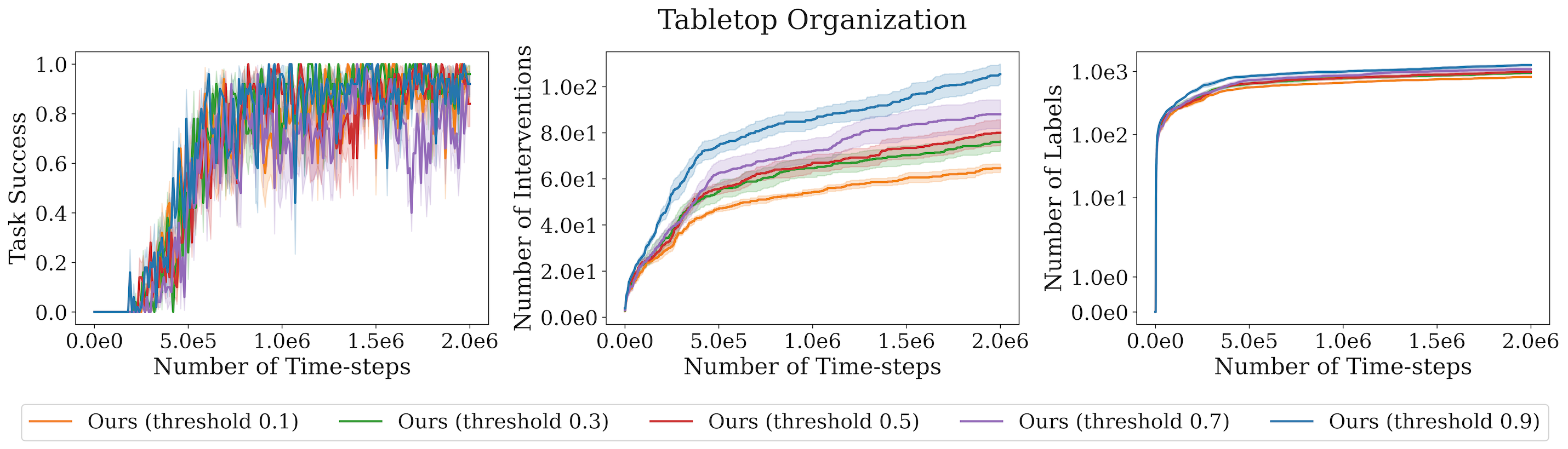} \\
    \vspace{0.3cm}
    \includegraphics[width=\linewidth]{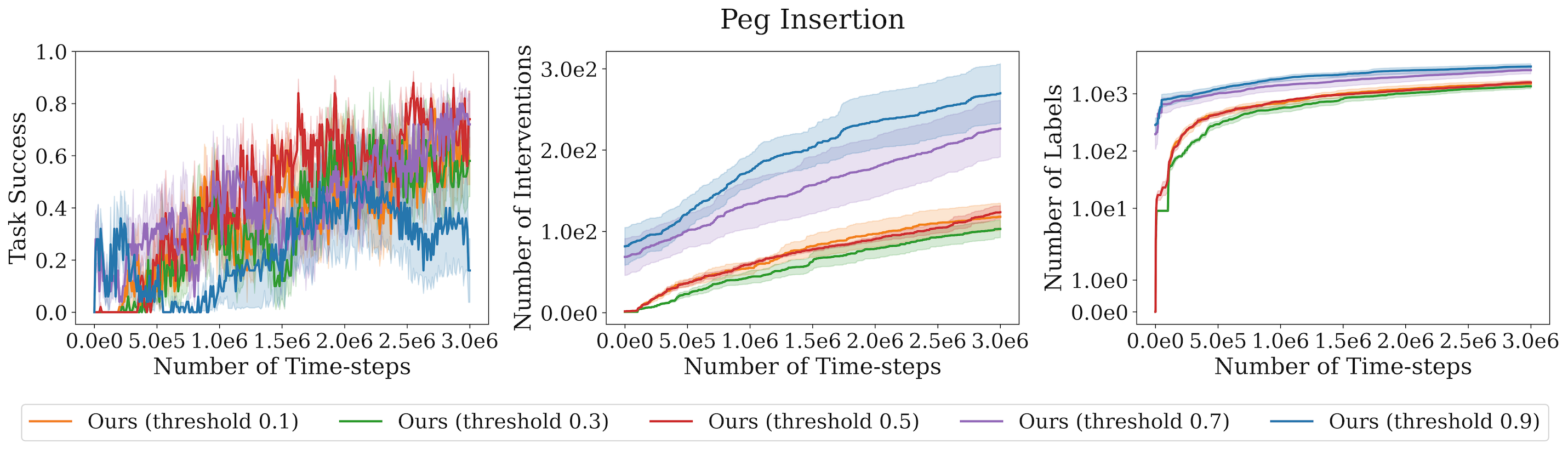} \\
    \vspace{0.3cm}
    \includegraphics[width=\linewidth]{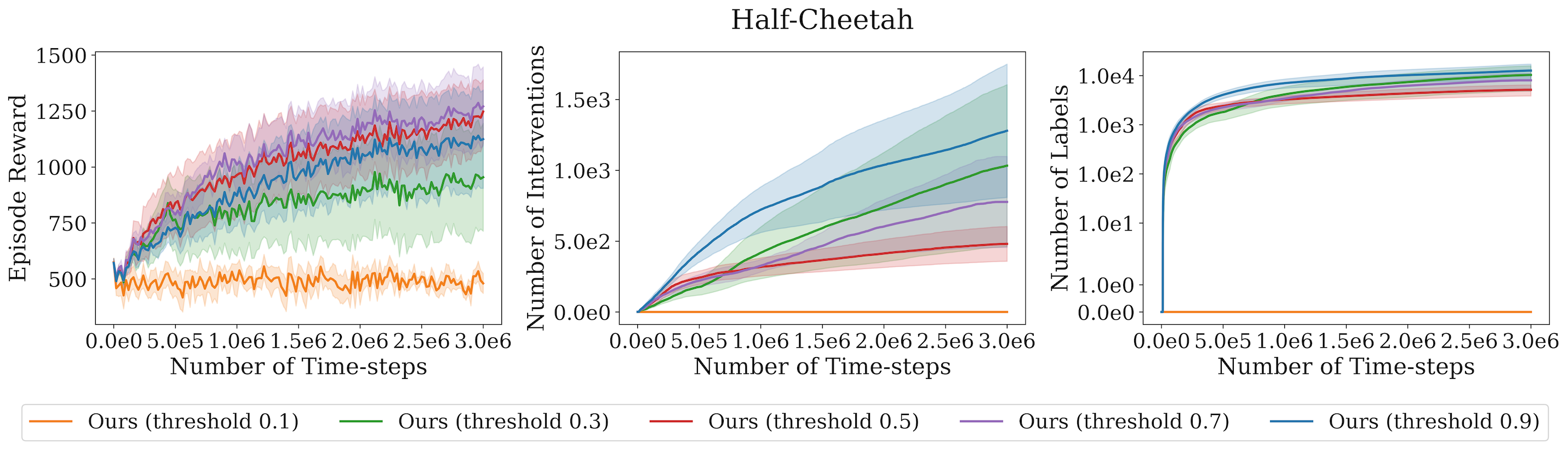}
    \caption{\small Task success (\textit{left}), number of interventions (\textit{middle}), and number of labels (\textit{right}) versus time for different threshold values.}
    \label{fig:threshold}
\end{figure}

\textbf{Reversibility classifier threshold.} The threshold that determines when to request an intervention can affect the performance of PAINT. Setting the threshold too low can result in the agent being stuck in irreversible states for a long time before calling for a reset while too high a threshold may trigger the reset too many times. To understand its sensitivity, we evaluated PAINT on the following threshold values, $\{0.1, 0.3, 0.5, 0.7, 0.9\}$, in the non-episodic tasks. The results, visualized in Fig.~\ref{fig:threshold}, indicate that PAINT with larger values like $0.9$ can perform worse, but across all tasks, a threshold value of $0.5$ yields high success rates while requiring few reset interventions.

\textbf{Reducing label queries with classifier confidence.} Rather than querying labels at every iteration of our binary search procedure, we study a variant of PAINT where we query labels from the expert only when there is uncertainty in the prediction. Specifically on the Maze environment, we study two forms of uncertainty: (a) the confidence of the classifier given by its output probability (predictions close to $0.5$ are uncertain) and (b) epistemic uncertainty captured by an ensemble of classifiers.
Only when the predictor is uncertain, i.e., $|\hat{\mathcal{R}}_\rho(s) - 0.5| < p$ in (a) and $\textsc{std}(\hat{\mathcal{R}}_\rho^i(s)) > p$ in (b), do we need to query the reversibility label. Otherwise, we can use the classifier's prediction in place of querying, which can further reduce the number of labels required by PAINT. 

In Fig.~\ref{fig:confidence} (top), we visualize the results of PAINT that queries based on the confidence of the classifier, when $|\hat{\mathcal{R}}_\rho(s) - 0.5| < p$ for $p \in \{ 0.1, 0.4 \}$, and we find that PAINT can achieve similar success rates even when it only queries labels for states with low-confidence predictions. In particular, PAINT with confidence-based querying requires fewer than $100$ labels compared to the $3$K labels required by standard PAINT. In Fig.~\ref{fig:confidence} (bottom), we report the results for PAINT with querying based on the uncertainty of an ensemble of classifiers, when $\textsc{std}(\hat{\mathcal{R}}_\rho^i(s)) > p$ for $p \in \{ 0.01, 0.001, 0.0001 \}$. Here, we find that a threshold of $0.0001$ achieves high success rates while only requiring around $750$ labels. These are promising improvements, and future work can extend them to more domains and more sophisticated querying strategies.

\begin{figure}
    \centering
    \includegraphics[width=\linewidth]{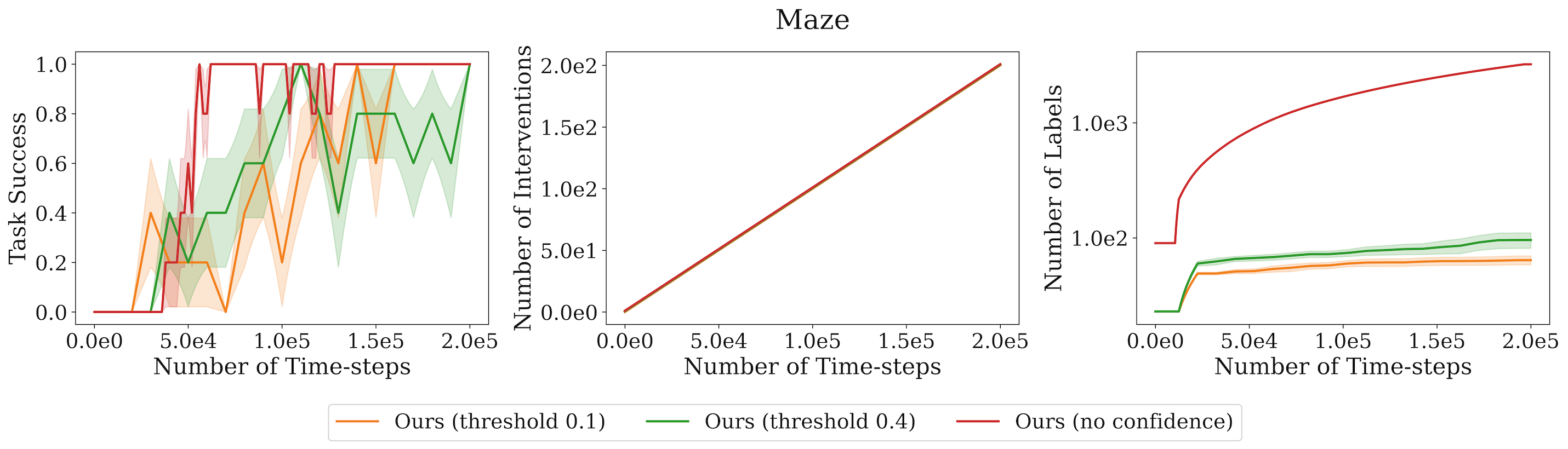}
    \includegraphics[width=\linewidth]{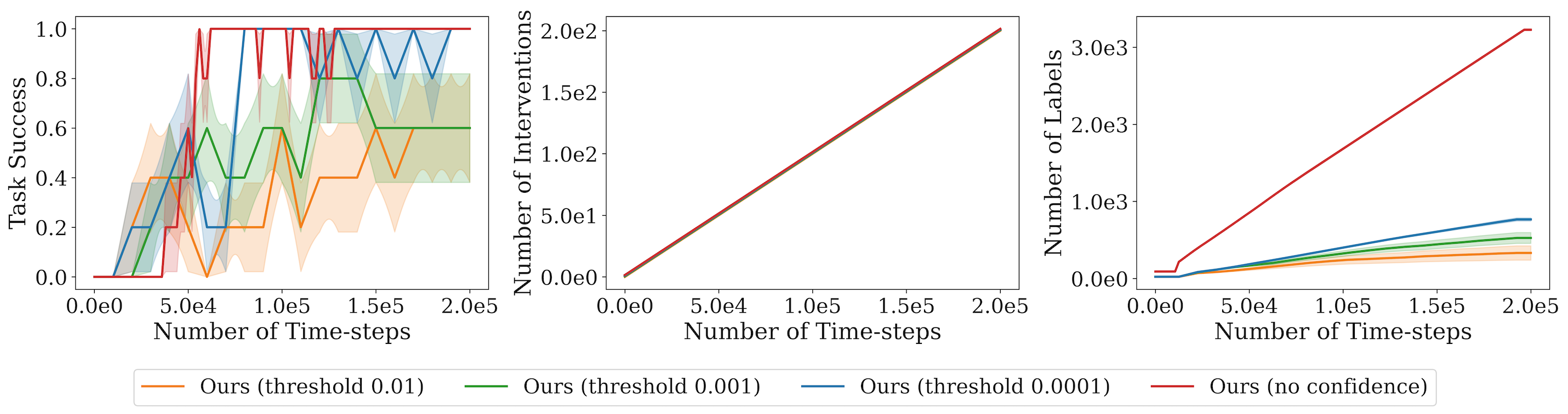}
    \caption{\small (\textit{top}) Querying based on the confidence of the classifier, i.e., when $|\hat{\mathcal{R}}_\rho(s) - 0.5| < p$ for $p \in \{ 0.1, 0.4 \}$. (\textit{bottom}) Querying based on the uncertainty of an ensemble of classifiers, i.e., when $\textsc{std}(\hat{\mathcal{R}}_\rho^i(s)) > p$ for $p \in \{ 0.01, 0.001, 0.0001 \}$.}
    \label{fig:confidence}
\end{figure}

\textbf{Noisy reversibility labels}. Thus far, our method has assumed the reversibility labels are noiseless. However, these labels in most practical settings come from human supervisors, who can inadvertently introduce noise into the labeling process. Hence, we simulate noisy reversibility labels in the maze environment, and design a \emph{robust} variant of our labeling scheme to account for possible noise in the labels. In the robust variant, in addition to querying the label for a state $s$, we also query the neighboring states, i.e., the sequence of $N$ states centered at $s$, and take the majority as the label for $s$.

We design two different noisy scenarios with (a) false positive labels and (b) false negative labels in the Maze environment. In (a), the trench regions can produce noisy labels. With probability $0.2$, the label is $1$, incorrectly labeling them as \emph{reversible}. In (b), the regions of the state space that neighbor the trenches can produce noisy labels. That is, with probability $0.2$, the label is $0$, incorrectly labeling them as \emph{irreversible}. In Fig.~\ref{fig:noisy_labels} (left and middle), we see that using a window size of $10$, PAINT succeeds $60\%$ of the time under false positive labels and $80\%$ under false negative labels, while smaller window sizes perform worse. The robust variant of PAINT can therefore tolerate some degree of label noise and only increase the number of reversibility labels required by a constant factor, thereby maintaining the same label complexity of $\mathcal{O}\left(N\log \left|\tau\right|_{\textrm{max}}\right)$ as before.

\begin{figure}
    \centering
    \includegraphics[height=1.65in]{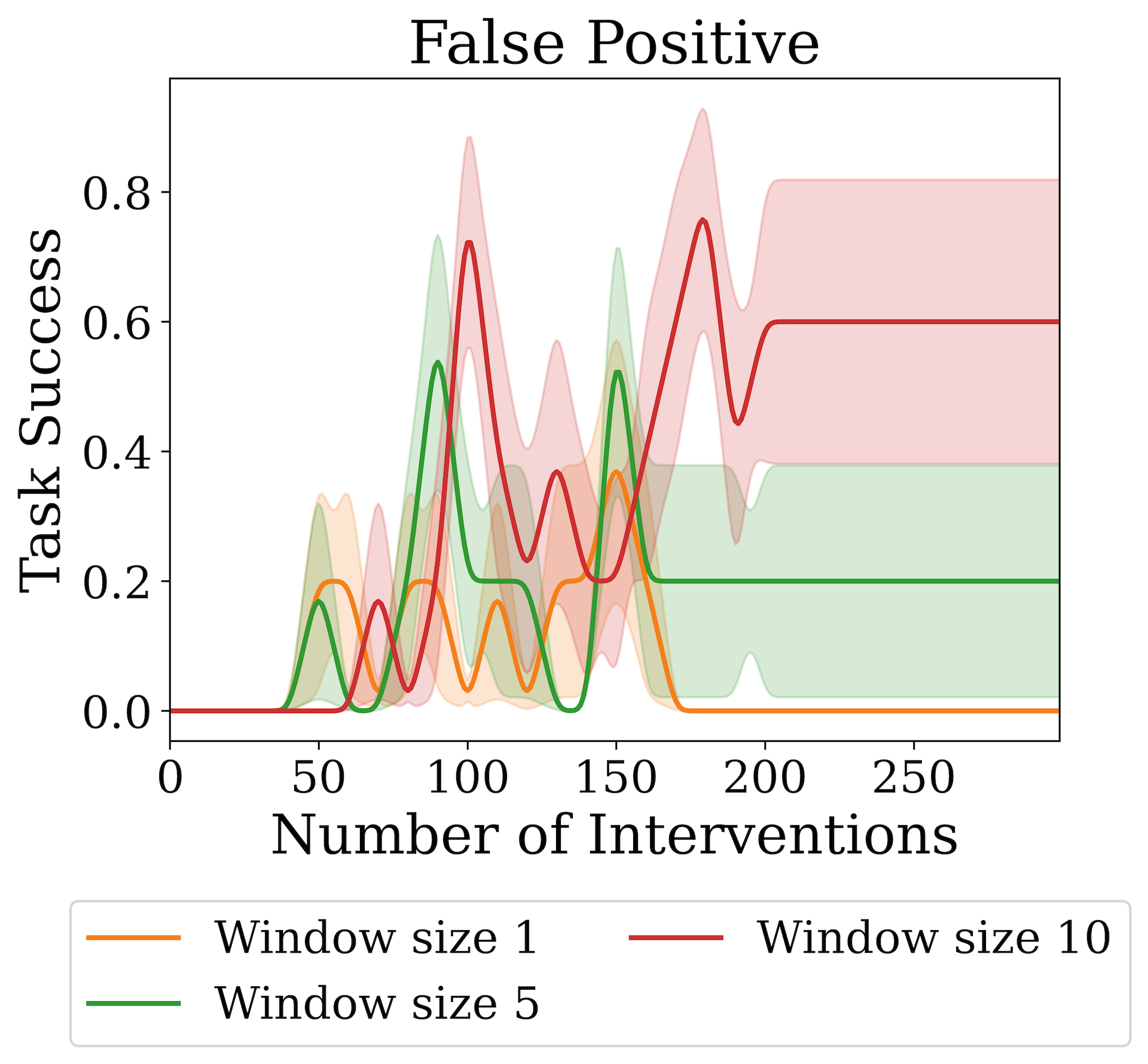}
    \includegraphics[height=1.65in]{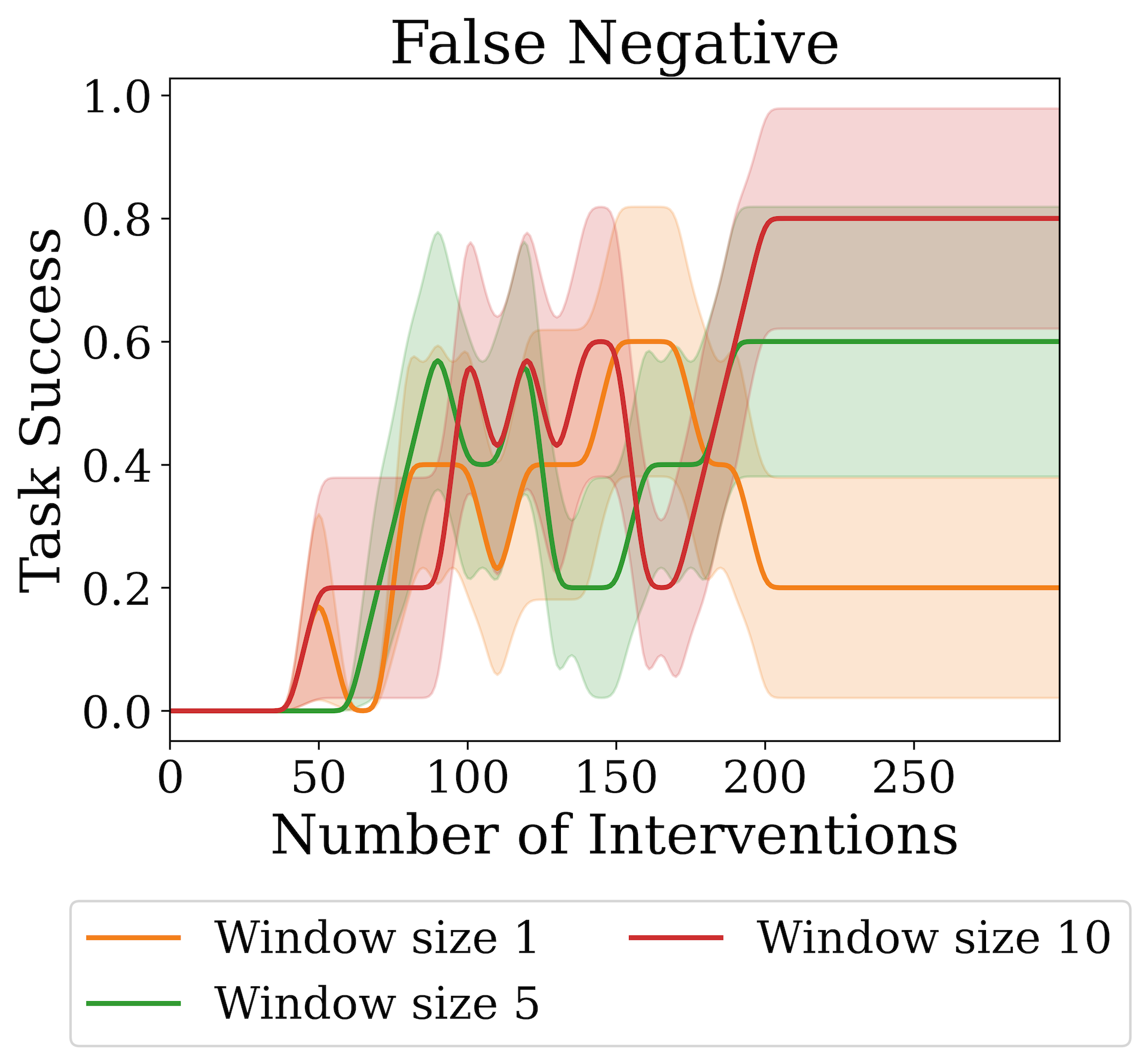}
    \includegraphics[height=1.65in]{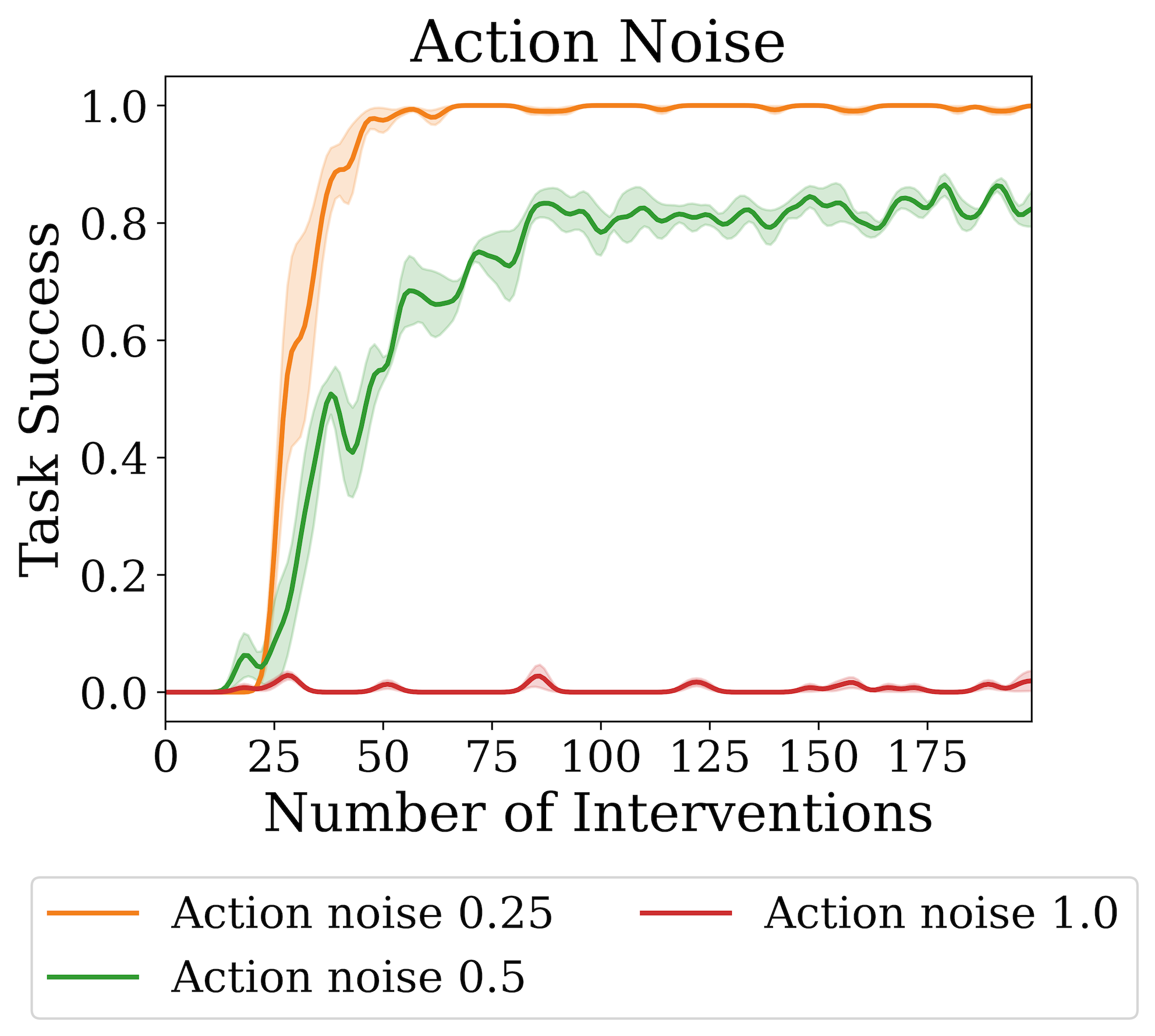}
    \caption{\small A robust variant of PAINT which queries neighboring states inside a window size of $1, 5,$ and $10$ with potential false positive labels (\textit{left}) and false negative labels (\textit{middle}). (\textit{right}) PAINT under varying noise added to the transitions.}
    \label{fig:noisy_labels}
\end{figure}

\textbf{Noisy transitions}. We also incorporate noise to the transition dynamics to further evaluate the robustness of PAINT. In the Maze environment, we add Gaussian noise of width $\sigma$ to the action before transitioning to the next state. The action space is $[-1.0, 1.0]^2$, and we evaluate on $\sigma$ values of $0.25, 0.5,$ and $1.0$. In Fig.~\ref{fig:noisy_labels} (right), we find that PAINT is robust up to a noise level of $0.5$. Unsurprisingly, when the noise is larger than the actions themselves, the agent fails to solve the task.

%% file: main.bbl
\begin{thebibliography}{53}
\providecommand{\natexlab}[1]{#1}
\providecommand{\url}[1]{\texttt{#1}}
\expandafter\ifx\csname urlstyle\endcsname\relax
  \providecommand{\doi}[1]{doi: #1}\else
  \providecommand{\doi}{doi: \begingroup \urlstyle{rm}\Url}\fi

\bibitem[Achiam et~al.(2017)Achiam, Held, Tamar, and
  Abbeel]{achiam2017constrained}
Joshua Achiam, David Held, Aviv Tamar, and Pieter Abbeel.
\newblock Constrained policy optimization.
\newblock In \emph{International conference on machine learning}, pages 22--31.
  PMLR, 2017.

\bibitem[Akrour et~al.(2011)Akrour, Schoenauer, and
  Sebag]{akrour2011preference}
Riad Akrour, Marc Schoenauer, and Michele Sebag.
\newblock Preference-based policy learning.
\newblock In \emph{Joint European Conference on Machine Learning and Knowledge
  Discovery in Databases}, pages 12--27. Springer, 2011.

\bibitem[Alshiekh et~al.(2018)Alshiekh, Bloem, Ehlers, K{\"o}nighofer, Niekum,
  and Topcu]{alshiekh2018safe}
Mohammed Alshiekh, Roderick Bloem, R{\"u}diger Ehlers, Bettina K{\"o}nighofer,
  Scott Niekum, and Ufuk Topcu.
\newblock Safe reinforcement learning via shielding.
\newblock In \emph{Proceedings of the AAAI Conference on Artificial
  Intelligence}, volume~32, 2018.

\bibitem[Arzate~Cruz and Igarashi(2020)]{arzate2020survey}
Christian Arzate~Cruz and Takeo Igarashi.
\newblock A survey on interactive reinforcement learning: design principles and
  open challenges.
\newblock In \emph{Proceedings of the 2020 ACM designing interactive systems
  conference}, pages 1195--1209, 2020.

\bibitem[Badia et~al.(2020)Badia, Sprechmann, Vitvitskyi, Guo, Piot,
  Kapturowski, Tieleman, Arjovsky, Pritzel, Bolt, et~al.]{badia2020never}
Adri{\`a}~Puigdom{\`e}nech Badia, Pablo Sprechmann, Alex Vitvitskyi, Daniel
  Guo, Bilal Piot, Steven Kapturowski, Olivier Tieleman, Mart{\'\i}n Arjovsky,
  Alexander Pritzel, Andew Bolt, et~al.
\newblock Never give up: Learning directed exploration strategies.
\newblock \emph{arXiv preprint arXiv:2002.06038}, 2020.

\bibitem[Bastani et~al.(2021)Bastani, Li, and Xu]{bastani2021safe}
Osbert Bastani, Shuo Li, and Anton Xu.
\newblock Safe reinforcement learning via statistical model predictive
  shielding.
\newblock In \emph{Robotics: Science and Systems}, 2021.

\bibitem[Bharadhwaj et~al.(2020)Bharadhwaj, Kumar, Rhinehart, Levine, Shkurti,
  and Garg]{bharadhwaj2020conservative}
Homanga Bharadhwaj, Aviral Kumar, Nicholas Rhinehart, Sergey Levine, Florian
  Shkurti, and Animesh Garg.
\newblock Conservative safety critics for exploration.
\newblock \emph{arXiv preprint arXiv:2010.14497}, 2020.

\bibitem[Biyik and Sadigh(2018)]{biyik2018batch}
Erdem Biyik and Dorsa Sadigh.
\newblock Batch active preference-based learning of reward functions.
\newblock In \emph{Conference on robot learning}, pages 519--528. PMLR, 2018.

\bibitem[Brockman et~al.(2016)Brockman, Cheung, Pettersson, Schneider,
  Schulman, Tang, and Zaremba]{brockman2016openai}
Greg Brockman, Vicki Cheung, Ludwig Pettersson, Jonas Schneider, John Schulman,
  Jie Tang, and Wojciech Zaremba.
\newblock Openai gym.
\newblock \emph{arXiv preprint arXiv:1606.01540}, 2016.

\bibitem[Chow et~al.(2017)Chow, Ghavamzadeh, Janson, and Pavone]{chow2017risk}
Yinlam Chow, Mohammad Ghavamzadeh, Lucas Janson, and Marco Pavone.
\newblock Risk-constrained reinforcement learning with percentile risk
  criteria.
\newblock \emph{The Journal of Machine Learning Research}, 18\penalty0
  (1):\penalty0 6070--6120, 2017.

\bibitem[Eysenbach et~al.(2017)Eysenbach, Gu, Ibarz, and
  Levine]{eysenbach2017leave}
Benjamin Eysenbach, Shixiang Gu, Julian Ibarz, and Sergey Levine.
\newblock Leave no trace: Learning to reset for safe and autonomous
  reinforcement learning.
\newblock \emph{arXiv preprint arXiv:1711.06782}, 2017.

\bibitem[Faulkner et~al.(2020)Faulkner, Short, and
  Thomaz]{faulkner2020interactive}
Taylor A~Kessler Faulkner, Elaine~Schaertl Short, and Andrea~L Thomaz.
\newblock Interactive reinforcement learning with inaccurate feedback.
\newblock In \emph{2020 IEEE International Conference on Robotics and
  Automation (ICRA)}, pages 7498--7504. IEEE, 2020.

\bibitem[Finn et~al.(2016)Finn, Tan, Duan, Darrell, Levine, and
  Abbeel]{finn2016deep}
Chelsea Finn, Xin~Yu Tan, Yan Duan, Trevor Darrell, Sergey Levine, and Pieter
  Abbeel.
\newblock Deep spatial autoencoders for visuomotor learning.
\newblock In \emph{2016 IEEE International Conference on Robotics and
  Automation (ICRA)}, pages 512--519. IEEE, 2016.

\bibitem[Ghadirzadeh et~al.(2017)Ghadirzadeh, Maki, Kragic, and
  Bj{\"o}rkman]{ghadirzadeh2017deep}
Ali Ghadirzadeh, Atsuto Maki, Danica Kragic, and M{\aa}rten Bj{\"o}rkman.
\newblock Deep predictive policy training using reinforcement learning.
\newblock In \emph{2017 IEEE/RSJ International Conference on Intelligent Robots
  and Systems (IROS)}, pages 2351--2358. IEEE, 2017.

\bibitem[Grinsztajn et~al.(2021)Grinsztajn, Ferret, Pietquin, Geist,
  et~al.]{grinsztajn2021there}
Nathan Grinsztajn, Johan Ferret, Olivier Pietquin, Matthieu Geist, et~al.
\newblock There is no turning back: A self-supervised approach for
  reversibility-aware reinforcement learning.
\newblock \emph{Advances in Neural Information Processing Systems}, 34, 2021.

\bibitem[Gu et~al.(2017)Gu, Holly, Lillicrap, and Levine]{gu2017deep}
Shixiang Gu, Ethan Holly, Timothy Lillicrap, and Sergey Levine.
\newblock Deep reinforcement learning for robotic manipulation with
  asynchronous off-policy updates.
\newblock In \emph{2017 IEEE international conference on robotics and
  automation (ICRA)}, pages 3389--3396. IEEE, 2017.

\bibitem[Gupta et~al.(2021)Gupta, Yu, Zhao, Kumar, Rovinsky, Xu, Devlin, and
  Levine]{Gupta2021ResetFreeRL}
Abhishek Gupta, Justin Yu, Tony Zhao, Vikash Kumar, Aaron Rovinsky, Kelvin Xu,
  Thomas Devlin, and Sergey Levine.
\newblock Reset-free reinforcement learning via multi-task learning: Learning
  dexterous manipulation behaviors without human intervention.
\newblock \emph{ArXiv}, abs/2104.11203, 2021.

\bibitem[Haarnoja et~al.(2018)Haarnoja, Zhou, Abbeel, and
  Levine]{haarnoja2018soft}
Tuomas Haarnoja, Aurick Zhou, Pieter Abbeel, and Sergey Levine.
\newblock Soft actor-critic: Off-policy maximum entropy deep reinforcement
  learning with a stochastic actor.
\newblock In \emph{International Conference on Machine Learning}, pages
  1861--1870. PMLR, 2018.

\bibitem[Han et~al.(2015)Han, Levine, and Abbeel]{han2015learning}
Weiqiao Han, Sergey Levine, and Pieter Abbeel.
\newblock Learning compound multi-step controllers under unknown dynamics.
\newblock In \emph{2015 IEEE/RSJ International Conference on Intelligent Robots
  and Systems (IROS)}, pages 6435--6442. IEEE, 2015.

\bibitem[Hoque et~al.(2021{\natexlab{a}})Hoque, Balakrishna, Novoseller,
  Wilcox, Brown, and Goldberg]{hoque2021thriftydagger}
Ryan Hoque, Ashwin Balakrishna, Ellen Novoseller, Albert Wilcox, Daniel~S
  Brown, and Ken Goldberg.
\newblock Thriftydagger: Budget-aware novelty and risk gating for interactive
  imitation learning.
\newblock \emph{arXiv preprint arXiv:2109.08273}, 2021{\natexlab{a}}.

\bibitem[Hoque et~al.(2021{\natexlab{b}})Hoque, Balakrishna, Putterman, Luo,
  Brown, Seita, Thananjeyan, Novoseller, and Goldberg]{hoque2021lazydagger}
Ryan Hoque, Ashwin Balakrishna, Carl Putterman, Michael Luo, Daniel~S Brown,
  Daniel Seita, Brijen Thananjeyan, Ellen Novoseller, and Ken Goldberg.
\newblock Lazydagger: Reducing context switching in interactive imitation
  learning.
\newblock In \emph{2021 IEEE 17th International Conference on Automation
  Science and Engineering (CASE)}, pages 502--509. IEEE, 2021{\natexlab{b}}.

\bibitem[Knox and Stone(2009)]{knox2009interactively}
W~Bradley Knox and Peter Stone.
\newblock Interactively shaping agents via human reinforcement: The tamer
  framework.
\newblock In \emph{Proceedings of the fifth international conference on
  Knowledge capture}, pages 9--16, 2009.

\bibitem[Krakovna et~al.(2018)Krakovna, Orseau, Kumar, Martic, and
  Legg]{krakovna2018penalizing}
Victoria Krakovna, Laurent Orseau, Ramana Kumar, Miljan Martic, and Shane Legg.
\newblock Penalizing side effects using stepwise relative reachability.
\newblock \emph{arXiv preprint arXiv:1806.01186}, 2018.

\bibitem[Kruusmaa et~al.(2007)Kruusmaa, Gavshin, and
  Eppendahl]{kruusmaa2007don}
Maarja Kruusmaa, Yuri Gavshin, and Adam Eppendahl.
\newblock Don't do things you can't undo: reversibility models for generating
  safe behaviours.
\newblock In \emph{Proceedings 2007 IEEE International Conference on Robotics
  and Automation}, pages 1134--1139. IEEE, 2007.

\bibitem[Lattimore and Szepesv{\'a}ri(2020)]{lattimore2020bandit}
Tor Lattimore and Csaba Szepesv{\'a}ri.
\newblock \emph{Bandit algorithms}.
\newblock Cambridge University Press, 2020.

\bibitem[Lee et~al.(2021)Lee, Smith, and Abbeel]{lee2021pebble}
Kimin Lee, Laura Smith, and Pieter Abbeel.
\newblock Pebble: Feedback-efficient interactive reinforcement learning via
  relabeling experience and unsupervised pre-training.
\newblock \emph{arXiv preprint arXiv:2106.05091}, 2021.

\bibitem[MacGlashan et~al.(2017)MacGlashan, Ho, Loftin, Peng, Wang, Roberts,
  Taylor, and Littman]{macglashan2017interactive}
James MacGlashan, Mark~K Ho, Robert Loftin, Bei Peng, Guan Wang, David~L
  Roberts, Matthew~E Taylor, and Michael~L Littman.
\newblock Interactive learning from policy-dependent human feedback.
\newblock In \emph{International Conference on Machine Learning}, pages
  2285--2294. PMLR, 2017.

\bibitem[Menda et~al.(2019)Menda, Driggs-Campbell, and
  Kochenderfer]{menda2019ensembledagger}
Kunal Menda, Katherine Driggs-Campbell, and Mykel~J Kochenderfer.
\newblock Ensembledagger: A bayesian approach to safe imitation learning.
\newblock In \emph{2019 IEEE/RSJ International Conference on Intelligent Robots
  and Systems (IROS)}, pages 5041--5048. IEEE, 2019.

\bibitem[Mnih et~al.(2015)Mnih, Kavukcuoglu, Silver, Rusu, Veness, Bellemare,
  Graves, Riedmiller, Fidjeland, Ostrovski, et~al.]{mnih2015human}
Volodymyr Mnih, Koray Kavukcuoglu, David Silver, Andrei~A Rusu, Joel Veness,
  Marc~G Bellemare, Alex Graves, Martin Riedmiller, Andreas~K Fidjeland, Georg
  Ostrovski, et~al.
\newblock Human-level control through deep reinforcement learning.
\newblock \emph{nature}, 518\penalty0 (7540):\penalty0 529--533, 2015.

\bibitem[Moldovan and Abbeel(2012)]{moldovan2012safe}
Teodor~Mihai Moldovan and Pieter Abbeel.
\newblock Safe exploration in markov decision processes.
\newblock \emph{arXiv preprint arXiv:1205.4810}, 2012.

\bibitem[Rahaman et~al.(2020)Rahaman, Wolf, Goyal, Remme, and
  Bengio]{rahaman2020learning}
Nasim Rahaman, Steffen Wolf, Anirudh Goyal, Roman Remme, and Yoshua Bengio.
\newblock Learning the arrow of time for problems in reinforcement learning.
\newblock 2020.

\bibitem[Sadigh et~al.(2017)Sadigh, Dragan, Sastry, and
  Seshia]{sadigh2017active}
Dorsa Sadigh, Anca~D Dragan, Shankar Sastry, and Sanjit~A Seshia.
\newblock \emph{Active preference-based learning of reward functions}.
\newblock 2017.

\bibitem[Savinov et~al.(2018)Savinov, Raichuk, Marinier, Vincent, Pollefeys,
  Lillicrap, and Gelly]{savinov2018episodic}
Nikolay Savinov, Anton Raichuk, Rapha{\"e}l Marinier, Damien Vincent, Marc
  Pollefeys, Timothy Lillicrap, and Sylvain Gelly.
\newblock Episodic curiosity through reachability.
\newblock \emph{arXiv preprint arXiv:1810.02274}, 2018.

\bibitem[Sharma et~al.(2021{\natexlab{a}})Sharma, Gupta, Levine, Hausman, and
  Finn]{sharma2021autonomouscurr}
Archit Sharma, Abhishek Gupta, Sergey Levine, Karol Hausman, and Chelsea Finn.
\newblock Autonomous reinforcement learning via subgoal curricula.
\newblock \emph{Advances in Neural Information Processing Systems}, 34,
  2021{\natexlab{a}}.

\bibitem[Sharma et~al.(2021{\natexlab{b}})Sharma, Xu, Sardana, Gupta, Hausman,
  Levine, and Finn]{sharma2021autonomous}
Archit Sharma, Kelvin Xu, Nikhil Sardana, Abhishek Gupta, Karol Hausman, Sergey
  Levine, and Chelsea Finn.
\newblock Autonomous reinforcement learning: Formalism and benchmarking.
\newblock \emph{arXiv preprint arXiv:2112.09605}, 2021{\natexlab{b}}.

\bibitem[Sharma et~al.(2022)Sharma, Ahmad, and Finn]{sharma2022state}
Archit Sharma, Rehaan Ahmad, and Chelsea Finn.
\newblock A state-distribution matching approach to non-episodic reinforcement
  learning.
\newblock \emph{arXiv preprint arXiv:2205.05212}, 2022.

\bibitem[Srinivasan et~al.(2020)Srinivasan, Eysenbach, Ha, Tan, and
  Finn]{srinivasan2020learning}
Krishnan Srinivasan, Benjamin Eysenbach, Sehoon Ha, Jie Tan, and Chelsea Finn.
\newblock Learning to be safe: Deep rl with a safety critic.
\newblock \emph{arXiv preprint arXiv:2010.14603}, 2020.

\bibitem[Sugiyama et~al.(2012)Sugiyama, Meguro, and
  Minami]{sugiyama2012preference}
Hiroaki Sugiyama, Toyomi Meguro, and Yasuhiro Minami.
\newblock Preference-learning based inverse reinforcement learning for dialog
  control.
\newblock In \emph{Thirteenth Annual Conference of the International Speech
  Communication Association}, 2012.

\bibitem[Tessler et~al.(2018)Tessler, Mankowitz, and Mannor]{tessler2018reward}
Chen Tessler, Daniel~J Mankowitz, and Shie Mannor.
\newblock Reward constrained policy optimization.
\newblock \emph{arXiv preprint arXiv:1805.11074}, 2018.

\bibitem[Thananjeyan et~al.(2020)Thananjeyan, Balakrishna, Nair, Luo,
  Srinivasan, Hwang, Gonzalez, Ibarz, Finn, and
  Goldberg]{thananjeyan2020recovery}
Brijen Thananjeyan, Ashwin Balakrishna, Suraj Nair, Michael Luo, Krishnan
  Srinivasan, Minho Hwang, Joseph~E Gonzalez, Julian Ibarz, Chelsea Finn, and
  Ken Goldberg.
\newblock Recovery rl: Safe reinforcement learning with learned recovery zones.
\newblock \emph{arXiv preprint arXiv:2010.15920}, 2020.

\bibitem[Thomas et~al.(2021)Thomas, Luo, and Ma]{thomas2021safe}
Garrett Thomas, Yuping Luo, and Tengyu Ma.
\newblock Safe reinforcement learning by imagining the near future.
\newblock \emph{Advances in Neural Information Processing Systems}, 34, 2021.

\bibitem[Turchetta et~al.(2020)Turchetta, Kolobov, Shah, Krause, and
  Agarwal]{turchetta2020safe}
Matteo Turchetta, Andrey Kolobov, Shital Shah, Andreas Krause, and Alekh
  Agarwal.
\newblock Safe reinforcement learning via curriculum induction.
\newblock \emph{Advances in Neural Information Processing Systems},
  33:\penalty0 12151--12162, 2020.

\bibitem[Wagener et~al.(2021)Wagener, Boots, and Cheng]{wagener2021safe}
Nolan~C Wagener, Byron Boots, and Ching-An Cheng.
\newblock Safe reinforcement learning using advantage-based intervention.
\newblock In \emph{International Conference on Machine Learning}, pages
  10630--10640. PMLR, 2021.

\bibitem[Wang et~al.(2022)Wang, Lee, Hakhamaneshi, Abbeel, and
  Laskin]{wang2022skill}
Xiaofei Wang, Kimin Lee, Kourosh Hakhamaneshi, Pieter Abbeel, and Michael
  Laskin.
\newblock Skill preferences: Learning to extract and execute robotic skills
  from human feedback.
\newblock In \emph{Conference on Robot Learning}, pages 1259--1268. PMLR, 2022.

\bibitem[Wang and Taylor(2018)]{wang2018interactive}
Zhaodong Wang and Matthew~E Taylor.
\newblock Interactive reinforcement learning with dynamic reuse of prior
  knowledge from human/agent's demonstration.
\newblock \emph{arXiv preprint arXiv:1805.04493}, 2018.

\bibitem[Wirth and F{\"u}rnkranz(2013)]{wirth2013preference}
Christian Wirth and Johannes F{\"u}rnkranz.
\newblock Preference-based reinforcement learning: A preliminary survey.
\newblock In \emph{Proceedings of the ECML/PKDD-13 Workshop on Reinforcement
  Learning from Generalized Feedback: Beyond Numeric Rewards}. Citeseer, 2013.

\bibitem[Xu et~al.(2020)Xu, Verma, Finn, and Levine]{Xu2020ContinualLO}
Kelvin Xu, Siddharth Verma, Chelsea Finn, and Sergey Levine.
\newblock Continual learning of control primitives: Skill discovery via
  reset-games.
\newblock \emph{ArXiv}, abs/2011.05286, 2020.

\bibitem[Yarats et~al.(2021{\natexlab{a}})Yarats, Fergus, Lazaric, and
  Pinto]{yarats2021drqv2}
Denis Yarats, Rob Fergus, Alessandro Lazaric, and Lerrel Pinto.
\newblock Mastering visual continuous control: Improved data-augmented
  reinforcement learning.
\newblock \emph{arXiv preprint arXiv:2107.09645}, 2021{\natexlab{a}}.

\bibitem[Yarats et~al.(2021{\natexlab{b}})Yarats, Kostrikov, and
  Fergus]{yarats2021image}
Denis Yarats, Ilya Kostrikov, and Rob Fergus.
\newblock Image augmentation is all you need: Regularizing deep reinforcement
  learning from pixels.
\newblock In \emph{International Conference on Learning Representations},
  2021{\natexlab{b}}.
\newblock URL \url{https://openreview.net/forum?id=GY6-6sTvGaf}.

\bibitem[Zanger et~al.(2021)Zanger, Daaboul, and Z{\"o}llner]{zanger2021safe}
Moritz~A Zanger, Karam Daaboul, and J~Marius Z{\"o}llner.
\newblock Safe continuous control with constrained model-based policy
  optimization.
\newblock In \emph{2021 IEEE/RSJ International Conference on Intelligent Robots
  and Systems (IROS)}, pages 3512--3519. IEEE, 2021.

\bibitem[Zhang and Cho(2016)]{zhang2016query}
Jiakai Zhang and Kyunghyun Cho.
\newblock Query-efficient imitation learning for end-to-end autonomous driving.
\newblock \emph{arXiv preprint arXiv:1605.06450}, 2016.

\bibitem[Zhu et~al.(2019)Zhu, Gupta, Rajeswaran, Levine, and
  Kumar]{zhu2019dexterous}
Henry Zhu, Abhishek Gupta, Aravind Rajeswaran, Sergey Levine, and Vikash Kumar.
\newblock Dexterous manipulation with deep reinforcement learning: Efficient,
  general, and low-cost.
\newblock In \emph{2019 International Conference on Robotics and Automation
  (ICRA)}, pages 3651--3657. IEEE, 2019.

\bibitem[Zhu et~al.(2020)Zhu, Yu, Gupta, Shah, Hartikainen, Singh, Kumar, and
  Levine]{zhu2020ingredients}
Henry Zhu, Justin Yu, Abhishek Gupta, Dhruv Shah, Kristian Hartikainen, Avi
  Singh, Vikash Kumar, and Sergey Levine.
\newblock The ingredients of real-world robotic reinforcement learning.
\newblock \emph{arXiv preprint arXiv:2004.12570}, 2020.

\end{thebibliography}
